\documentclass[11pt,letter]{article}
\usepackage{FW2020c}

\title{A Framework of Learning through Empirical Gain Maximization}

\author[1]{Yunlong Feng}
\author[2]{Qiang Wu}
\affil[1]{Department of Mathematics and Statistics, University at Albany}
\affil[2]{Department of Mathematical Sciences, Middle Tennessee State University}

\date{}

\begin{document}
\maketitle
    
\begin{abstract}

\noindent We develop in this paper a framework of empirical gain maximization (EGM) to address the robust regression problem where heavy-tailed noise or outliers may be present in the response variable. The idea of EGM is to approximate the density function of the noise distribution  instead of approximating the truth function directly as usual. Unlike the classical maximum likelihood estimation that encourages equal importance of all observations and could be problematic in the presence of abnormal observations, EGM schemes can be interpreted from a minimum distance estimation viewpoint and allow the ignorance of those observations. Furthermore, it is shown that several well-known robust nonconvex regression paradigms, such as Tukey regression and truncated least square regression, can be reformulated into this new framework. We then develop a learning theory for EGM, by means of which a unified analysis can be conducted for these well-established but not fully-understood regression approaches. Resulting from the new framework, a novel interpretation of existing bounded nonconvex loss functions can be concluded. Within this new framework, the two seemingly irrelevant terminologies, the well-known Tukey's biweight loss for robust regression and the triweight kernel for nonparametric smoothing, are closely related. More precisely, it is shown that the Tukey's biweight loss can be derived from the triweight kernel. Similarly, other frequently employed bounded nonconvex loss functions in machine learning such as the truncated square loss, the Geman-McClure loss, and the exponential squared loss can also be reformulated from certain smoothing kernels in statistics. In addition, the new framework enables us to devise new bounded nonconvex loss functions for robust learning.
\end{abstract}

\section{Introduction, Motivation, and Preliminaries}

In this paper, we are concerned with robust regression problems where conditional distributions may have heavier-than-Gaussian tails or be contaminated by outliers. In machine learning 
and statistics, regression procedures are typically carried out through Empirical Risk Minimization (ERM) or its variants. Denote $X$ as the input variable taking values in $\mathcal{X}\subset\mathbb{R}^d$ and $Y$ the output variable taking values in $\mathbb{R}$. Given $n$ i.i.d observations $\mathbf{z}=\{(x_i,y_i)\}_{i=1}^n$ and choosing a hypothesis space $\mathcal{H}$, ERM can be formulated as
\begin{align}\label{ERM_general}
f_\mathbf{z}=\arg\min_{f\in\mathcal{H}}\frac{1}{n}\sum_{i=1}^n \ell(y_i-f(x_i)),
\end{align}
where $\ell:\mathbb{R}\rightarrow \mathbb{R}_+$ is a loss function that measures the point-wise goodness-of-fit when using $f(x)$ to predict $y$. Several frequently employed loss functions for regression with continuous output include the square loss, the absolute deviation loss, as well as the check loss. The resulting ERM schemes learn, respectively, the conditional mean, the conditional median, and the conditional quantile function. Under the additive noise model
\begin{align}\label{data_generating_model}
Y=f^\star(X)+\varepsilon,
\end{align}
where $\varepsilon$ denotes the noise variable, these ERM schemes can also be deduced from maximum likelihood estimation (MLE) by assuming a Gaussian, Laplace, or asymmetry Laplace prior distribution of the noise variable, respectively. In this sense, these ERM schemes are essentially MLE-based, though in practice the prior distributional assumptions may not be imposed explicitly or may even be abandoned. To deal with robust regression problems, various M-estimators such as the Tukey regression estimator and the truncated least square estimator, which can be viewed as generalized maximum likelihood estimators, are proposed in robust statistics. Carrying over from robust estimation of location parameters in parametric statistics, the idea of M-estimation is to implement MLE based on heavy-tailed noise distributions. In the machine learning context, the nonparametric counterparts of these M-estimators are also frequently used for robust prediction, especially in the area of computer vision. 

While maximum likelihood estimates are efficient and often lead to effective regression estimators, their disadvantages are also obvious. To explain this, let us consider again the $n$-size sample $\mathbf{z}$ generated by the regression model \eqref{data_generating_model}. Denoting $p_\mathsmaller{\varepsilon|X}$ as the density of the noise variable $\varepsilon$ conditioned on $X$, the likelihood of drawing the sample $\mathbf{z}$ given the parameter ``$f$'' from the ``parameter space'' $\mathcal{H}$ is
$\mathcal{L}(\mathbf{z};f)=\prod_{i=1}^n p_\mathsmaller{\varepsilon|x_i}(y_i-f(x_i))$.
In MLE, one aims at seeking the optimal  ``parameter'' $f_\mathbf{z}$ in $\mathcal{H}$ such that this likelihood is maximized, i.e., 
\begin{align}\label{MLE}
f_\mathbf{z}:=\arg\max_{f\in\mathcal{H}}\prod_{i=1}^n p_\mathsmaller{\varepsilon|x_i}(y_i-f(x_i)).
\end{align}
In other words, the purpose of searching for the optimal ``parameter'' $f_\mathbf{z}$ is achieved by maximizing the product of the likelihood functions of all the observations. Taking a logarithm operation to the product leads to the log-likelihood $\log\mathcal{L}(\mathbf{z};f)$ and further brings us ERM schemes induced by the square loss and the least absolute deviation loss when the noise is assumed to be Gaussian and Laplace, respectively. MLE is known to be (asymptotically) more efficient than other estimators when the noise distribution is correctly specified. However, in the machine learning context,  the premise that the noise distribution is known \textit{a priori} is seldom the case and more often than not, is not realistic. On the other hand, one of the shortcomings of MLE lies in that it is sensitive to abnormal observations. This could be understood intuitively from the formula \eqref{MLE}, where the maximization of the product obviously encourages equal importance of all the observations. Specifically, the maximization does not allow the zeroness of the likelihood of any observation no matter whether the observation is abnormal or not. This could be problematic in practice where the acquired data are frequently contaminated. To address the non-robustness problem of MLE based methodologies, tremendous efforts have been made in the literature of statistics, machine learning, as well as other data science areas.

\subsection{Problem Formulation}\label{problem_setup}

Our study in this paper starts from a family of robust learning approaches of the  form
\begin{align}\label{EGM_initial}
f_\mathbf{z}=\arg\max_{f\in\mathcal{H}}\frac{1}{n}\sum_{i=1}^n p_\mathsmaller{\varepsilon|x_i}(y_i-f(x_i)),
\end{align}
where $\mathcal{H}$ is the hypothesis space and $p_\mathsmaller{\varepsilon|X}$ the density of the noise $\varepsilon|X$. Since $p_\mathsmaller{\varepsilon|X}$ is unknown in practice, we consider its surrogate $p_\sigma$ by ignoring the dependence of $\varepsilon$ on $X$ and the norming constant. Then, the scheme can be generalized to the form 
\begin{align}\label{EGM_general}
f_{\mathbf{z},\sigma}=\arg\max_{f\in\mathcal{H}}\frac{1}{n}\sum_{i=1}^n p_\sigma(y_i-f(x_i)),
\end{align}
where we call $p_\sigma:\mathbb{R}\rightarrow\mathbb{R}_+$ a \textit{gain function} if it is unimodal, attains its peak value at $0$, and integrable. Correspondingly, the scheme \eqref{EGM_general} is termed as \emph{Empirical Gain Maximization (EGM)}, a systematic investigation of which will be the main focus of this paper.  

\subsection{Motivating Scenarios}\label{subsec::motivating_scenarios}
Our development of the EGM framework and the introduction of gain function are inspired by the following motivating scenarios that find tremendous real-world applications in robust estimation across numerous data science fields. 
\medskip 

\begin{description}

\item[Motivating Scenario I: Tukey Regression.]
The well-known Tukey regression based on Tukey's biweight loss,  
proposed in \cite{beaton1974fitting}, can be equivalently formulated as an EGM scheme \eqref{EGM_general} associated with the gain function
	\begin{align*}
	p_\mathsmaller{\sigma}(t)=\left(1-\frac{t^2}{\sigma^2}\right)^3\mathbb{I}_{\{|t|\leq \sigma\}},
	\end{align*}
which we call \emph{triweight gain} function. Here and in what follows, for a set $S$ (or an event), $\mathbb{I}_S$ represents an indicator function that takes the value $1$ on $S$ (or when $S$ occurs) and $0$ otherwise. Tukey regression has found numerous applications in a great variety of data science areas where robustness is a concern; see e.g., \cite{green1984iteratively,meer1991robust,can1999robust,stewart1999estimating,bramati2007robust,belagiannis2015robust,chang2018robust,clarkson2019dimensionality}.

\item[Motivating Scenario II: Truncated Least Square.] 
The truncated least square regression can be equivalently formulated as an EGM scheme 
associated with the \emph{Epanechnikov gain} function
	\begin{align*}
	p_\mathsmaller{\sigma}(t)=\left(1-\frac{t^2}{\sigma^2}\right)\mathbb{I}_{\{|t|\leq \sigma\}}.
	\end{align*}
The related studies and applications of truncated least squares regression, to name a few, can be found in \cite{hinich1975simple,yang1995breakdown,ikami2018fast,lauer2018exact,liu2019minimizing}.

\item[Motivating Scenario III: Geman-McClure Regression.] 
Introduced in \cite{geman1985bayesian} for image analysis, the Geman-McClure regression can be formulated as 
an EGM scheme 
with the \emph{Cauchy gain} function $$p_\sigma(t)=\frac{\sigma^2}{\sigma^2+t^2}.$$ 
It has been applied extensively especially to the area of robust computer vision, see e.g., \cite{black1996unification,de1998view,nikou1998robust,yacoob1999tracking,yacoob2000learned,bar2006image,shah2017robust,chatterjee2017robust,jiang2020coherent}.

\item[Motivating Scenario IV: Maximum Correntropy.]
The maximum correntropy criterion \cite{liu2007correntropy, principe2010information}, motivated by maximizing the information gain measured by correntropy between the input variable $X$ and output variable $Y$, is equivalent to an EGM scheme 
with the \emph{Gaussian gain} function
\begin{align*}
p_\sigma(t)=\exp\left(-\frac{t^2}{2\sigma^2}\right).
\end{align*}
This gain function was also proposed as a goodness-of-fit measurement in various contexts of robust estimation under different terminologies such as the Welsch's loss \cite{dennis1978techniques}, 
the inverted Gaussian loss \cite{kording2004loss}, the exponential squared loss \cite{{wang2013robust}}, and the reflected normal loss \cite{spiring1993reflected}.  
Its theoretical properties were recently investigated in \cite{feng2018learning,fengWu2019learning,feng2017statistical,feng2020new} from a statistical learning viewpoint. 
\end{description}

\medskip 

Further details of the above gain functions and their correspondence to bounded nonconvex loss functions will be discussed in Section \ref{sec::formu_EGM} below.
What are common behind the above four robust learning schemes are that (1) all of them can be naturally formulated into the EGM framework \eqref{EGM_general} with the associated gain function either the kernel of a common probability distribution or a common smoothing kernel; (2) they are all introduced in order to pursue robustness in estimation procedures with a scale parameter $\sigma$ controlling the robustness; (3) they all have extremely wide applications in robust estimation problems from science to technology. However, in a distribution-free setup, their learning performance, especially the relationship between their learnability, robustness, and the scale parameter, has not been well assessed. In this sense, they are well-established learning schemes but not fully-understood ones; and (4) they were usually treated as  M-estimators and interpreted from an ERM viewpoint or further traced to MLE. However, such interpretations, on one hand, cannot reveal the working mechanism of the resulting estimators and so cannot fully explain their robustness, on the other hand, cannot help discover new learning schemes of the same kind.

\subsection{Objectives and Contributions}

By introducing gain functions and developing an EGM framework, the objectives of this study are not solely to develop new robust learning schemes but rather to (1) present a unified analysis of the several well-known robust regression estimators and study their performance from a learning theory viewpoint; (2) pursue novel insights and interpretations of these estimators, which could help explain their robustness merits; (3) develop a learning theory framework as well as novel machineries for analyzing  robust regression schemes that fall into the same vein; and (4) devise and explore more new robust regression estimators that are of the same kind. It turns out that our newly developed framework can fulfill these objectives effectively.         
Our main contributions are summarized as follows:
\begin{itemize}
\item Motivated by several widely used robust regression schemes, we introduce gain function as an alternate measurement of goodness-of-fit in regression and develop an EGM framework. It is shown that maximizing empirical gain in regression can be viewed as maximizing the summation of the likelihood functions. Unlike the product of likelihood functions in MLE, summation does not encourage the equal importance of all observations but allows even the zeroness of the likelihoods of abnormal ones. This observation may better explain the robustness of these regression schemes.   

\item Interestingly, the newly developed framework subsumes a bunch of well-known robust regression schemes, especially those listed in the motivating scenarios, such as the celebrated Tukey regression, the classic truncated least square regression, and the widely employed Geman-McClure regression in computer vision. We stress here that under the new framework, Tukey regression estimators can be obtained via EGM induced by triweight gain function, which comes from the triweight kernel; truncated least square estimators can be deduced from EGM associated with Epanechnikov gain function, which comes from the Epanechnikov kernel; and Geman-McClure regression estimators can be derived from EGM with Cauchy gain function, which results from the Cauchy kernel. Such findings provide us an alternative understanding of these robust regression schemes. That is, these M-estimators may be interpreted more naturally as minimum distance estimators. 

\item We conduct a unified learning theory analysis of EGM schemes. In the learning theory literature, convex regression schemes have been extensively studied. However, studies and assessments of nonconvex ones, such as Tukey regression, truncated least square regression, and Geman-McClure regression, are still sparse though they have been extensively applied. Our study provides a unified learning theory assessment of these robust regression schemes. More specifically, we consider two different setups, namely, learning without noise misspecification and distribution-free learning. We show that when learning without noise misspecification, EGM estimators are regression calibrated while in a distribution-free setup, under certain conditions, EGM estimators can learn the underlying truth function under weak moment conditions.     

\item The present study brings us novel insights into existing bounded nonconvex loss functions. Furthermore, the correspondence between bounded nonconvex losses and gain functions allow us to introduce more new bounded nonconvex losses that enjoy similar properties. For instance, a generalized Tukey's loss may be formulated to cater to further needs in robust regression as well as classification problems.
\end{itemize}

\subsection{Roadmap and Notation}

The roadmap of this paper is as follows. In Section~\ref{sec::first_look}, we provide some first look at EGM by comparing it with ERM and interpreting it from a minimum distance estimation viewpoint. Section~\ref{sec::formu_EGM} exposes the way of defining gain functions and formulating EGM.  Gain functions are exampled and categorized in this section. In particular, the correspondence between some representative bounded nonconvex losses and gain functions is also illustrated here. In Section~\ref{sec::sober_look}, we look into EGM schemes by investigating fundamental questions in EGM based learning and conducting a unified learning theory analysis. We assess the learning performance of EGM estimators in different scenarios. Specifically, as an instantiation, we show that one can directly apply the developed theory to those well-established robust learning schemes mentioned in the motivating scenarios. Some further insights and perspectives are provided in Section~\ref{sec::lessons_insights}. The paper is concluded in Section~\ref{sec::conclusion}. Intermediate lemmas and proofs of the theorems are provided in the appendix.   
\smallskip 

Throughout this paper, $\|\cdot\|_{2,\rho}^2$  denotes the $L^2$-norm induced by the marginal distribution $\rho_{\!_X}$.  $\mathcal C(\mathcal X)$ denotes the space of bounded continuous functions on $\mathcal X.$  $p_{\!_{Y|X}}$ stands for the conditional density of $Y$ conditioned on $X$ and $p_{\!_{Y|X=x}}$, or $p_{\!_{Y|x}}$ for brevity, represents the conditional density of $Y$ conditioned on $X=x$. We also denote  $a\lesssim b$ if  $a\leq c b$ for some absolute constant $c>0$.

\section{A First Look at Empirical Gain Maximization}\label{sec::first_look}
In this section, at an intuitive level, we present a first look at EGM by comparing it with the classical ERM. To this end, we first investigate situations where ERM based regression schemes may fail and then look into EGM by interpreting it from a \textit{minimum distance estimation} (MDE) as well as a \textit{maximum likelihood estimation} (MLE) viewpoint. 

\subsection{EGM vs. ERM}
As stated earlier, ERM based regression schemes can be traced to the MLE framework where one maximizes the product of the likelihoods of all the observations. That is,  
\begin{align}\label{ERM_MLE} 
f_\mathbf{z}=\arg\max_{f\in\mathcal{H}}\frac{1}{n}\prod_{i=1}^n p_\mathsmaller{\varepsilon|x_i}(y_i-f(x_i)).
\end{align} 
In \eqref{ERM_MLE}, for any $f\in\mathcal{H}$, maximizing the product of the likelihood functions encourages the equal importance of all the residuals $y_i-f(x_i)$, $i=1,\cdots,n$, and does not tolerate the zeroness of any likelihood no matter whether the residual is caused by an abnormal observation or not. In contrast, in EGM regression schemes, one maximizes the summation of the likelihood functions of the residuals. That is,  
\begin{align}\label{EGM_MDE} 
f_\mathbf{z}=\arg\max_{f\in\mathcal{H}}\frac{1}{n}\sum_{i=1}^n p_\mathsmaller{\varepsilon|x_i}(y_i-f(x_i)).
\end{align}
Intuitively, maximizing the summation instead of the product allows small values or even zeroness of likelihood of certain residuals and so may significantly reduce the impact of those abnormal ones. Therefore, EGM may outperform ERM in terms of robustness to abnormal observations.

\subsection{EGM: When MLE Meets MDE}
To proceed with our comparison, for any measurable function $f:\mathcal{X}\rightarrow \mathbb{R}$, we denote $E_f$ as the random variable defined by the residual between $Y$ and $f(X)$, i.e., $E_f=Y-f(X)$. Then, for any fixed $f$ and any realization of $X$, say $x$, the density function of $E_f|x$ can be obtained by translating that of $\varepsilon|x$ horizontally $f^\star(x)-f(x)$ units. Consequently, the density of $E_f|x$ can be expressed as
\begin{align*}
p_\mathsmaller{E_f|x}(t)=p_\mathsmaller{\varepsilon|x}(t+f(x)-f^\star(x)).
\end{align*}
Similarly, we also have
\begin{align*}
p_\mathsmaller{\varepsilon|x}(t)=p_\mathsmaller{E_f|x}(t+f^\star(x)-f(x)).
\end{align*}
Moreover, reminded by \cite{fan2014consistency}, we know that  
\begin{align*}
p_\mathsmaller{E_f}(t)=\int_{\mathcal{X}}p_\mathsmaller{\varepsilon|x}(t+f(x)-f^\star(x))\mathrm{d}\rho_X(x)
\end{align*}
defines a density of the random variable $E_f$, and 
\begin{align*}
p_\mathsmaller{\varepsilon}(t)=\int_{\mathcal{X}}p_\mathsmaller{E_f|x}(t+f^\star(x)-f(x))\mathrm{d}\rho_X(x)
\end{align*}
defines a density of the random variable $\varepsilon$. 

Continuing our discussion in the introduction, the population version of the M-estimators resulting from the ERM scheme \eqref{ERM_general} and so the MLE \eqref{MLE} can be equivalently cast as 
\begin{align*}
\widehat{f_\mathcal{H}}=\arg\max_{f\in\mathcal{H}}\mathbb{E}\log p_\mathsmaller{\varepsilon|X}(Y-f(X)).
\end{align*}
Simple computations show that one also has 
\begin{align*}
\widehat{f_\mathcal{H}}=\arg\min_{f\in\mathcal{H}}{\sf{KL}}\Big(p_\mathsmaller{\varepsilon|X}(Y-f(X)), \,p_{\mathsmaller{E_f|X}}(Y-f(X))\Big),
\end{align*}
where ${\sf{KL}}(p_1,p_2)$ denotes the KL-divergence between the two distributions $p_1$ and $p_2$ and measures their discrepancy.

On the other hand, recall that the EGM estimator \eqref{EGM_general} originates from \eqref{EGM_initial}, which can be treated as an M-estimator. Therefore,  the EGM scheme can also be interpreted from an MLE viewpoint. Meanwhile, note that the population counterpart of \eqref{EGM_initial} is 
\begin{align*}
f_\mathcal{H}=\arg\max_{f\in\mathcal{H}}\mathbb{E}p_\mathsmaller{\varepsilon|X}(Y-f(X)).  
\end{align*}
For any measurable functions $f$, we define the integrated squared density-based distance, ${\sf{dist}}(p_\mathsmaller{E_f},p_\mathsmaller{\varepsilon})$, between $p_\mathsmaller{E_f}$ and $p_\mathsmaller{\varepsilon}$ as
	\begin{align*}
	{\sf{dist}}(p_\mathsmaller{E_f},p_\mathsmaller{\varepsilon}):=\sqrt{\int_\mathcal{X}\int_{-\infty}^{+\infty}(p_\mathsmaller{E_f|x}(t)-p_\mathsmaller{\varepsilon|x}(t))^2\mathrm{d}t\,\mathrm{d}\rho_X(x)}.
	\end{align*}
It is easy to see that the above distance between $p_\mathsmaller{E_f}$ and $p_\mathsmaller{\varepsilon}$ defines a metric between $p_\mathsmaller{E_f}$ and $p_\mathsmaller{\varepsilon}$.  In particular, if $f$ equals $f^\star$ on $\mathcal{X}$, then we have ${\sf{dist}}(p_\mathsmaller{E_f},p_\mathsmaller{\varepsilon})=0$. Notice that the data-generating model $Y=f^\star(X)+\varepsilon$ defines a location family, which implies that 
\begin{align*}
f_\mathcal{H}=\arg\min_{f\in\mathcal{H}}{\sf{dist}}(p_\mathsmaller{E_f},p_\mathsmaller{\varepsilon}).
\end{align*}

Therefore, EGM can be interpreted both from an MLE viewpoint and from an MDE viewpoint. To be brief, when MLE meets MDE, EGM estimators come into sight. Because of such dual interpretations, they may possess built-in merits such as robustness delivered by MDE and fast convergence rates inherited from MLE, as we shall explore later.  

\section{Formulating Empirical Gain Maximization and Gain Functions}\label{sec::formu_EGM}
In this section, we first detail the reasoning process that leads to EGM \eqref{EGM_general} and also the definition of gain functions. More gain functions will be exampled and categorized and their correspondence with bounded nonconvex losses will also be discussed.

\subsection{Towards Making EGM based Learning Applicable}\label{subsec::applicable}
While the initial scheme \eqref{EGM_initial} seems to be promising due to its merits in terms of robustness, barriers may be encountered in its implementation, the primary one of which is tractability. This is because in practice $p_\mathsmaller{\varepsilon}$ is unknown in advance. Therefore, further efforts are needed to address this problem and make it practically applicable. One way of tackling this problem, inspired by existing studies in learning theory, is to find a tractable relaxation, which is the gain function $p_\mathsmaller{\sigma}$ defined in Section \ref{problem_setup}. Listed below are several possible strategies that may be adopted for this purpose.

\smallskip 

\begin{description}

\item[Assuming the Density.]
Recall that when MLE is used to estimate the conditional mean, the Gaussianity of the noise is usually assumed which leads to the least squares regression. 
If Laplace noise is assumed, then one comes to the least absolute deviation estimation which approaches the conditional median in regression. When moving attention to EGM,
following the same way, one may also assume that the noise variable $\varepsilon$ obeys a certain law such as the Gaussian or Laplace. In particular, with Laplace noise, we then have
$p_\mathsmaller{\sigma}(y_i-f(x_i))=\exp\left(-|y_i-f(x_i)|/\sigma\right)$.
Under mild conditions, it can be shown that the resulting EGM estimator approaches the conditional median or the conditional mode function with properly chosen $\sigma$ values. Similarly, with Gaussian noise, one gets
$p_\mathsmaller{\sigma}(y_i-f(x_i))=\exp\left(-(y_i-f(x_i))^2/2\sigma^2\right)$
and the resulting EGM scheme can be equivalently formulated as the maximum correntropy criterion stated in Motivating Scenario IV. 

We arrive at the EGM scheme \eqref{EGM_general} from the scheme \eqref{EGM_initial} by choosing $p_\sigma$ which is an assumed density and serves as a surrogate of the true noise density. However, we remark that, similar to MLE, EGM may still be practically effective even if the assumed noise density deviates from the ground truth, as shown in Section \ref{sec::sober_look} below.

\smallskip 

\item[Learning the Density.] Noticing that in \eqref{EGM_initial}, for any fixed $f$ in $\mathcal{H}$, we maximize the summation of the values of the unknown density $p_\mathsmaller{\varepsilon}$ at $n$ points $y_i-f(x_i)$, $i=1,\cdots,n$. When a learning machine $f$ is utilized, the observations $y_i-f(x_i)$, $i=1,\cdots,n$, can be treated as realizations of the unknown noise variable defined by the residual $Y-f^\star(X)$. Therefore, for each $i$, one could estimate the point-wise density $p_\mathsmaller{\varepsilon}(y_i-f(x_i))$ by using observations $y_j-f(x_j)$, $j=1,\cdots,i-1,i+1,\cdots, n$, by means of the \textit{Parzen window} density estimator. Explicitly, let $K_\sigma$ be a smoothing kernel with the bandwidth parameter $\sigma$, under mild conditions, one has the following estimate of $p_\mathsmaller{\varepsilon}$ that serves as its relaxation
\begin{align*}
p_\mathsmaller{\sigma}(y_i-f(x_i))= \frac{1}{n-1}\sum_{{j=1}\atop{j\neq i}}^n K_\sigma(y_i-f(x_i)-y_j+f(x_j)).
\end{align*}  
The resulting EGM scheme can be formulated as 
\begin{align*}
f_{\mathbf{z},\sigma}=\arg\max_{f\in\mathcal{H}}\frac{1}{n(n-1)}\sum_{i=1}^n \sum_{{j=1}\atop{j\neq i}}^n K_\sigma(y_i-f(x_i)-y_j+f(x_j)).
\end{align*}
Canonical smoothing kernels include Gaussian kernel, Laplace kernel, Epanechnikov kernel, uniform kernel, triangle kernel, etc. Interestingly, the above learning scheme is essentially equivalent to the one induced by the minimum error entropy algorithm \cite{erdogmus2000comparison,principe2010information} that was recently investigated from a learning theory viewpoint in \cite{hu2013learning,fan2014consistency,hu2016convergence,hu2019distributed, guo2020distributed}. 
 
\smallskip

\item[Approximating the Density.] Another way of dealing with the unknown density $p_\mathsmaller{\varepsilon}$ is that one may directly approximate this density function. To this end, we recall that $\varepsilon:= Y-f^\star(X)$ is a one-dimensional random variable. As a mild assumption, one may assume that $p_\mathsmaller{\varepsilon}$ is continuous on $\mathbb{R}$. From approximation theory, we know that one can approximate this one-dimensional continuous function $p_\mathsmaller{\varepsilon}$ arbitrarily well by using certain basis functions on $\mathbb{R}$. As an example, one may use the convex combination of the one-dimensional Gaussian kernel, which leads to
\begin{align*}
p_\mathsmaller{\sigma}(y_i-f(x_i))=\sum_{j=1}^K w_j \exp\left(-\frac{(y_i-f(x_i))^2}{\sigma_j^2}\right),
\end{align*}
where the coefficients $w_j$ are positive constants such that $\sum_{j=1}^K w_j=1$, $K\geq 1$ a positive integer, and $\sigma_j>0$ for $j=1, \cdots, K$. Then, the resulting EGM is 
\begin{align*}
f_{\mathbf{z},\sigma}=\arg\max_{f\in\mathcal{H}}\frac{1}{n}\sum_{i=1}^n\sum_{j=1}^K w_j \exp\left(-\frac{(y_i-f(x_i))^2}{\sigma_j^2}\right).
\end{align*}
In particular, if $K=1$, it reduces Gaussian EGM. We note that similar ideas have been investigated for robust learning, see e.g., \cite{chen2018mixture}.
\end{description}

In addition to the above-mentioned approaches to finding tractable relaxations of $p_\mathsmaller{\varepsilon}$, one may also use smoothing kernels from statistics 
since each smoothing kernel defines a density. By stretching or compressing a smoothing kernel vertically or horizontally, one may approximate the unknown density $p_\mathsmaller{\varepsilon}$. For illustration, we will provide more examples in the next subsection. It would be interesting to  explore further techniques for finding such a relaxation.

\subsection{Gain Function: Formal Definition and More Examples}\label{subsec::definition}
With the preparations above, we are now ready to introduce a formal definition of gain functions, which leads to the general EGM formulation \eqref{EGM_general}. 

\begin{definition}[Gain function]\label{definition::gain_function}
	A function $p_\mathsmaller{\sigma}:\mathbb{R}\rightarrow \mathbb{R}_+$ with a parameter $\sigma>0$ is said to be a gain function if there exists a generating function $\phi:\mathbb{R}\rightarrow\mathbb{R}_+$  such that $p_\mathsmaller{\sigma}(t)=\phi\left(t/\sigma\right)$ for any $t\in\mathbb{R}$ and the following conditions are satisfied:
\begin{enumerate}[label=\emph{(\arabic*)}]
		\item $0<\int_{-\infty}^{+\infty} \phi(t)\mathrm{d}t<+\infty;$ and 
		\item $\phi$ is  non-decreasing on $(-\infty,0]$ and non-increasing on $[0,+\infty)$. 
\end{enumerate} 	
\end{definition}

The gain function $p_\sigma$ is introduced as a surrogate of  $p_{\varepsilon}$. The scale parameter $\sigma$ is used to stretch or compress the function $\phi$ so as to approximate the density. According to the definition, gain functions attain their peak values at the point $0$. An intuitive explanation of this restriction is that one gains the most if a learning machine $f$ fits $y$ exactly at the point $x$.  

\begin{remark}
The terminology ``gain function" in Definition \ref{definition::gain_function} has been used in various disciplines. For instance, in game theory, gain function is better known as ``\textit{pay-off function}". It is a function defined on the set of situations in a game, the values of which are a numerical description of the utility of a player or of a team of players in a given situation. In economics, gain function is better known as ``\textit{utility function}". It is a function that measures preferences over a set of goods and services. Its value represents the satisfaction that consumers receive for choosing and consuming a product or service. In the present study, gain function is not referred to as the ones in game theory or economics, though it may be related to those terminologies. The introduction of gain function here is directly inspired by the studies in \cite{weiss1984estimation,weiss1987estimating} for robust statistical estimation.
\end{remark} 
 
Following Definition \ref{definition::gain_function} and the discussions in Section \ref{subsec::applicable}, one can immediately write out a variety of gain functions. 

\begin{example}[Triweight gain function]\label{Tukey_gain}
When taking the triweight kernel as a gain function, we have the 
triweight gain function
	\begin{align*}
	p_\mathsmaller{\sigma}(t)=\left(1-\frac{t^2}{\sigma^2}\right)^3\mathbb{I}_{\{|t|\leq \sigma\}}.
	\end{align*}
\end{example}

\begin{example}[Epanechnikov gain function]\label{Epanechnikov_gain}
When taking the Epanechnikov kernel as a gain function, we come to the Epanechnikov gain function
	\begin{align*}
	p_\mathsmaller{\sigma}(t)=\left(1-\frac{t^2}{\sigma^2}\right)\mathbb{I}_{\{|t|\leq \sigma\}}.
	\end{align*}
\end{example}

\begin{example}[Cauchy gain function]\label{Cauchy_gain}
	When considering the kernel of a Cauchy distribution with the location parameter $0$, we obtain the 
	Cauchy gain function
	\begin{align*}
	p_\mathsmaller{\sigma}(t)=\frac{\sigma^2}{\sigma^2+t^2}.
	\end{align*}
\end{example}

\begin{example}[Gaussian gain function]\label{Gaussian_gain}
	When considering the kernel of a standard Gaussian distribution, we have the 
	Gaussian gain function
	\begin{align*}
	p_\mathsmaller{\sigma}(t)=\exp\left(-\frac{t^2}{\sigma^2}\right).
	\end{align*}
\end{example}

\begin{example}[Laplace gain function]\label{laplace_gain}
Considering the kernel of a Laplace distribution with the location parameter $0$, we have the 
Laplace gain function
	\begin{align*}
	p_\mathsmaller{\sigma}(t)=\exp\left(-\frac{|t|}{\sigma}\right).
	\end{align*}
\end{example}

\begin{example}[Cosine gain function]\label{cosine_gain}
Using the cosine kernel as a gain function leads to the 
Cosine gain function
\begin{align*}
p_\sigma(t)=\cos\left(\frac{\pi t}{2\sigma}\right)\mathbb{I}_{\{|t|\leq \sigma\}}.
\end{align*}
\end{example}

\begin{example}[Uniform gain function]\label{uniform_gain}
Using the uniform kernel as a gain function gives the uniform gain function
\begin{align*}
p_\sigma(t)=\frac{1}{2\sigma}\mathbb{I}_{\{|t|\leq \sigma\}}.
\end{align*}
\end{example}

\subsection{Categorizing Gain Functions}

As shown above, a variety of gain functions can be introduced in various ways for different purposes. For instance, Gaussian and Cauchy gain functions may be employed in EGM to learn the conditional mean function in regression, while by means of the Laplace gain function one may learn the conditional median function. In this part, we make efforts to categorize gain functions by defining type $\alpha$ gain functions and strongly mean-calibrated gain functions. 

\begin{definition}
A gain function $p_\sigma$ is said to be of \emph{type $\alpha$} if there exist two nonnegative constants $\alpha$ and $c$ such that   
\begin{align*}
p_\sigma(t)= p_\sigma(0)-c\left(\frac{|t|}{\sigma}\right)^\alpha + R_\alpha\left(\frac{|t|}{\sigma}\right),
\end{align*}
where the remainder term $R_\alpha\left(\frac{|t|}{\sigma}\right)=o\left(\frac{|t|^\alpha}{\sigma^\alpha}\right)$ as $\frac{|t|}{\sigma}\to 0.$ In particular, $p_\sigma$ is said to be of \emph{exact type $\alpha$} if $R_\alpha\left(\frac{|t|}{\sigma}\right)=0$ for any
$\frac{|t|}{\sigma}\le 1.$  
\end{definition}

It is easy to verify that 
\begin{itemize}
    \item  the triweight, Cauchy, Gaussian, and Cosine gain functions in Examples \ref{Tukey_gain}, \ref{Cauchy_gain}, \ref{Gaussian_gain}, and \ref{cosine_gain} 
    are of type $2$;
   \item the Epanechnikov gain function in Example \ref{Epanechnikov_gain} is of exact type $2$;
    \item the Laplace gain function in Example \ref{laplace_gain} is of type $1$; and
    \item the uniform gain function in Example \ref{uniform_gain} is of exact type $0$.
\end{itemize}

Intuitively, a type $2$ gain function may be used for mean regression while a type $1$ gain function is 
for median regression. As mean regression will be the main focus in what follows, among type $2$ gain functions, we are particularly interested in strongly mean-calibrated ones as well as exactly mean-calibrated ones defined below.

\begin{definition}\label{mean_calibrated_gain}
A gain function $p_\sigma$ is said to be \emph{strongly mean-calibrated} if there exist a representing function $\psi:\mathbb{R}_+\to \mathbb{R}_+$ and absolute positive constants $L_1$ and $L_2$ such that $p_\sigma(t):=\psi(t^2/\sigma^2)$ and the following two conditions hold
\begin{enumerate}[label=\emph{(\arabic*)}]
       \item $\psi(t^2)$ is $L_1$-Lipschitz w.r.t. $t$ on $\mathbb{R}$; and
    \item  $\psi^\prime(0)<0$ and $\psi^\prime(t)$ exists and is $L_2$-Lipschitz 
    on $[0,1)$.
\end{enumerate}
In particular, a strongly mean-calibrated gain function $p_\sigma$ is said to be \emph{exactly mean-calibrated} if $\psi^\prime(t)$ is a constant function on $(0,1)$.
\end{definition}

As per the above definition, the type $2$ gain functions listed in Examples  \ref{Tukey_gain}-\ref{Gaussian_gain} and \ref{cosine_gain} can be further categorized as in Table \ref{type_2_categorization}. Clearly, not all type $2$ gain functions are strongly mean-calibrated. The following proposition further reveals the relations between strongly mean-calibrated and (exact) type $2$ gain functions.  

\begin{table}[t]
	\setlength{\tabcolsep}{12pt}
	\setlength{\abovecaptionskip}{5pt}
	\setlength{\belowcaptionskip}{5pt}
	\centering
	\captionsetup{justification=centering}
	\vspace{.5em}
	\begin{tabular}{@{\hspace{2pt}}lcccc @{\hspace{2pt}}}
		\toprule
	\textbf{Gain Function}	&  \textbf{Mean-Calibration} & $\psi(t)$ & $L_1$ &  $L_2$\\ \midrule
		\quad	Triweight   &  Strong  & $(1-t)^3\mathbb I_{\{|t|\le 1\}}$ & $\frac{96}{5\sqrt{5}}$ & $6$  \\ \midrule
	   \quad Epanechnikov  &  Exact & $(1-t)\mathbb I_{\{|t|\le 1\}}$  & $2$ &  $0$\\ \midrule
	  \quad  Cauchy & Strong  & $\frac 1{1+t}$ & $\frac{3\sqrt{3}}{8}$ &  $2$\\  \midrule 
	   	\quad   Gaussian & Strong  & $\mathrm e^{-t/2}$ & $\mathrm e^{-1/2}$ &   $\frac{1}{4}$\\ \midrule 
	   	   \quad Cosine &  Strong & $\cos\left(\frac{\pi\sqrt{t}}{2}\right) \mathbb I_{\{|t|\leq 1\}}  $ &   $\pi$ &  $\frac{\pi^4}{192}$ \\
 		\bottomrule
	\end{tabular}
	\caption{Further Categorizations of Type $2$ Gain Functions in Examples \ref{Tukey_gain}-\ref{uniform_gain}}\label{type_2_categorization} 
\end{table}

\begin{proposition}
A strongly mean-calibrated gain function must be of type $2$, and an exactly mean-calibrated gain function must be of exact type $2$. 
\end{proposition}
\begin{proof}
Let $p_\sigma$ be a strongly mean-calibrated gain function. According to Definition \ref{mean_calibrated_gain}, there exists a representing function $\psi$ such that $p_\sigma(t)=\psi(t^2/\sigma^2)$. Applying the mean value theorem, we know that for any $0<u<1$, it holds that 
\begin{align*}
\psi(u)-\psi(0)=\psi^\prime(\xi)u,    
\end{align*}
where $0<\xi<u$. To show that $p_\sigma$ is of type $2$, it suffices to prove that there exists a positive constant $c$ such that 
\begin{align*}
\psi^\prime(\xi)u+cu= o(u), \quad \hbox{as}\,\, u\to 0. 
\end{align*}
This is obvious if we set $c=-\psi^\prime(0)>0$ and recall that   $\psi^\prime(u)$ is $L_2$-Lipschitz on $[0,1)$. Replacing $u$ with $t^2/\sigma^2$, we have proved that $p_\sigma$ is of type $2$. The conclusion that an exactly mean-calibrated gain function must be of exact type $2$ is obvious. 
\end{proof}

As we shall see later, EGM schemes induced by strongly mean-calibrated gain functions are asymptotically mean calibrated in regression and their sharp error bounds can be established. It should be remarked that the conditions for strongly mean-calibrated gain functions are sufficient to ensure the (asymptotic) mean calibration properties of the resulting EGM estimators and to establish fast convergence rates. In fact, such conditions can be relaxed to much weaker ones if one is only interested in regression consistency. Likewise, one can also further categorize type $1$ gain functions and investigate their behaviors in median regression, which is beyond the scope of the present study.  

\subsection{Interpreting Bounded Nonconvex Losses as Gain Functions}\label{translate_bounded_loss}
In recent years, bounded nonconvex loss functions are playing more and more important roles in machine learning applications especially in computer vision as it is commonly accepted that the boundedness of a loss function is essential for outlier resistance, see e.g., \cite{long2010random,she2011outlier}. Several canonical bounded nonconvex losses that are frequently employed in robust estimation problems include truncated square loss, Tukey's biweight loss, Geman-McClure loss, exponential squared loss, and Andrews loss. Interestingly, within the EGM framework, these bounded nonconvex losses can be naturally interpreted as gain functions. Such correspondences  are detailed below and are also summarized in Table \ref{correspondences_gain_loss}.

\begin{exampleprime} [Tukey's biweight loss] \label{Tukeys_loss}
The well-known Tukey's biweight loss for robust estimation was introduced in \cite{tukey1960survey} and is defined as 
\begin{align*}
\ell_\sigma(t)=
\begin{cases}
\frac{\sigma^2}{6}\left[1-\left(1-\frac{t^2}{\sigma^2}\right)^3\right],& \hbox{if}\quad |t|\leq \sigma;\\
\frac{\sigma^2}{6},&  \hbox{otherwise}.
\end{cases}
\end{align*}  
It can be deduced from the triweight gain function in Example \ref{Tukey_gain} and leads to the triweight EGM in Motivating Scenario I.  
\end{exampleprime}

\begin{exampleprime}
[Truncated square loss]\label{skipped_mean}
The truncated square loss proposed in \cite{hinich1975simple} is given as 
\begin{align*}
\ell_\sigma(t)=\min\{t^2,\sigma^2\}.
\end{align*}
It is also known as the skipped mean loss or Talwar loss. The truncated square loss can be translated from the Epanechnikov gain function in Example \ref{Epanechnikov_gain} and leads to the Epanechnikov EGM in Motivating Scenario II.
\end{exampleprime}

\begin{exampleprime}
[Geman-McClure loss]\label{GM_loss}
The Geman-McClure loss proposed in \cite{geman1985bayesian} is defined as follows
\begin{align*}
\ell_\sigma(t)= \frac{t^2}{\sigma^2+t^2}.  
\end{align*}
Clearly, it can be derived from the  Cauchy gain function in Example \ref{Cauchy_gain} and leads to the Geman-McClure regression in Motivating Scenario III. 
\end{exampleprime}

\begin{exampleprime}
[Exponential squared loss]\label{inverted_gaussian_loss} 
The exponential squared loss is defined as
\begin{align*}
\ell_\sigma(t)= \sigma^2 \left(1-\exp\left(-\frac{t^2}{2\sigma^2}\right)\right).
\end{align*}
It can be derived from the Gaussian gain function in Example \ref{Gaussian_gain} and leads to the Gaussian EGM in Motivating Scenario IV.
\end{exampleprime}

\begin{exampleprime}
[Exponential absolute loss]\label{exponential_absolute_loss} 
The exponential absolute loss studied in \cite{leonard2001bayesian} and \cite{chen2015robust} is defined as 
\begin{align*}
\ell_\sigma(t)= 1-\exp\left(-\frac{|t|}{\sigma}\right).    
\end{align*}	
It can be derived from the Laplace gain function in Example~\ref{laplace_gain}.
\end{exampleprime}

\begin{exampleprime}
[Andrews loss]\label{Andrews_loss}
Andrews loss was introduced in \cite{andrews1974robust} and has been applied in robust estimation in statistics and machine learning. It takes the form 
\begin{align*}
\ell_\mathsmaller{\sigma}(t)=
\begin{cases}
\sigma^2\left(1-\cos(\frac{\pi t}{2\sigma})\right), & \hbox{if}\,\,\,|t|\leq \sigma;\\
\sigma^2,\quad & \hbox{otherwise},
\end{cases}
\end{align*}
and can be derived from the Cosine gain function in Example \ref{cosine_gain}.
\end{exampleprime}

\begin{exampleprime}
[Box loss]\label{box_loss}
Box loss takes the following form 
\begin{align*}
\ell_\mathsmaller{\sigma}(t)=
\begin{cases}
0,\quad \hbox{if}\,\,\,|t|\leq \sigma;\\
1,\quad \hbox{otherwise}.
\end{cases}
\end{align*}
Corresponding to the uniform gain function in Example \ref{uniform_gain}, this loss function was employed to perform modal regression in \cite{lee1989mode} and the resulting EGM scheme also gives the maximum consensus problem in computer vision, see e.g., \cite{chin2019robust,chin2017maximum}. 
\end{exampleprime}

\begin{table}[t]
	\setlength{\tabcolsep}{12pt}
	\setlength{\abovecaptionskip}{5pt}
	\setlength{\belowcaptionskip}{5pt}
	\centering
	\captionsetup{justification=centering}
	\vspace{.5em}
	\begin{tabular}{@{\hspace{2pt}}llc@{\hspace{2pt}}}
		\toprule
	\textbf{Bounded Nonconvex Loss}	&  \textbf{Gain Function} & \textbf{Related Examples}\\ \midrule
			Tukey's biweight loss &  \quad  Triweight   &  \ref{Tukey_gain} and \ref{Tukeys_loss}\\ \midrule
	   Truncated square loss &  \quad Epanechnikov   &  \ref{Epanechnikov_gain} and \ref{skipped_mean}\\ \midrule
	    Geman-McClure loss & \quad Cauchy  &  \ref{Cauchy_gain} and \ref{GM_loss}\\  \midrule 
	   	   Exponential squared loss & \quad Gaussian   &  \ref{Gaussian_gain} and \ref{inverted_gaussian_loss}\\ \midrule 
	   	      Exponential absolute loss &\quad Laplace  &  \ref{laplace_gain} and \ref{exponential_absolute_loss} \\ \midrule
Andrews loss &  \quad Cosine   &  \ref{cosine_gain} and \ref{Andrews_loss}\\
\midrule 
Box loss &  \quad Uniform   &  \ref{uniform_gain} and \ref{box_loss}\\
\bottomrule
	\end{tabular}
	\caption{Correspondence between Gain Functions and Existing Bounded Nonconvex Losses}\label{correspondences_gain_loss}
\end{table}

Following the same way, one can translate more bounded nonconvex losses, such as the well-known Hampel's loss \cite{hampel1974influence}, into gain functions. Interestingly, such translations bridge the two seemingly irrelevant sets of tools, e.g., the triweight kernel and the Tukey's biweight loss, the Epanechnikov kernel and the truncated square loss, the Cauchy density function and the Geman-McClure loss. One benefit of such translations is that they allow us to interpret those ERM based estimators more naturally from a minimum distance estimation viewpoint, which can help explain their robustness.

\section{A Sober Look at Empirical Gain Maximization}\label{sec::sober_look}
In this section, we take a sober look at EGM. To this end, we first propose several fundamental questions that are raised when assessing EGM from a learning theory viewpoint. We then assess the performance of EGM estimators in two different setups, namely, the distribution-free setup and the setup where the noise distribution is correctly specified. We also conduct case studies by applying our theoretical results to the motivating scenarios.

\subsection{Fundamental Questions in Learning with EGM}

Recalling that in the ERM scheme \eqref{ERM_general}, in order to assess the out-of-sample prediction ability of $f_\mathbf{z}$, one evaluates the \textit{excess generalization error}
\begin{align*}
\mathbb{E}\ell(Y-f_{\mathbf{z}}(X))-\mathbb{E}\ell(Y-f_{\mathsmaller{\mathcal{M}},\,\ell}(X)),
\end{align*} 
where the expectation is taken jointly w.r.t. $X$ and $Y$ and
\begin{align*}
f_{\mathsmaller{\mathcal{M}},\,\ell}=\arg\min_{f\in\mathcal{M}}\mathbb{E}\ell(Y-f(X))
\end{align*}
serves as the oracle of the ERM scheme induced by the loss function $\ell$. When the loss function $\ell$ is chosen as the square loss, one has  $f_{\mathsmaller{\mathcal{M}},\,\ell}=f^\star$ with $f^\star$ being the conditional mean function. In particular, in this case, one also has the following relation
\begin{align*}
\|f_\mathbf{z}-f^\star\|_{2,\rho}^2=\mathbb{E}\ell(Y-f_\mathbf{z}(X))-\mathbb{E}\ell(Y-f^\star(X)).
\end{align*}
When $\ell$ is a general convex loss, under certain noise assumptions, one may still be able to characterize the oracle $f_{\mathsmaller{\mathcal{M}},\,\ell}$ and further show that the oracle is the underlying truth $f^\star$. Moreover, the convergence   $\mathbb{E}\ell(Y-f_\mathbf{z}(X))\rightarrow \mathbb{E}\ell(Y-f^\star(X))$ may also imply the convergence of $\|f_\mathbf{z}- f^\star\|_{2,\rho}^2$. In the statistical learning literature, the related studies that address the above concerns for ERM schemes induced by convex losses have been conducted extensively and theoretical frameworks have been well developed, see e.g., \cite{cucker2007learning,steinwart2008support}.

However, for EGM schemes, the story becomes more complicated due to the nonconcavity of gain functions and the involvement of the parameter $\sigma$. For any measurable function $f:\mathcal{X}\rightarrow\mathbb{R}$, we denote its \textit{generalization gain} associated with the gain function $p_\mathsmaller{\sigma}$ as   
$$\mathcal{G}_\sigma(f)=\mathbb{E}p_\mathsmaller{\sigma}(Y-f(X))$$ 
and refer to the empirical counterpart  
$$\mathcal{G}_{\sigma,\mathbf{z}}(f)=\frac{1}{n}\sum_{i=1}^n p_\mathsmaller{\sigma}(y_i-f(x_i))$$
as its \textit{empirical gain}. Moreover, we denote the quantity
$\mathcal{G}_\sigma(f_\mathsmaller{\mathcal{M},\sigma})-\mathcal{G}_\sigma(f)$
as the \textit{excess generalization gain} of $f$, where
$f_\mathsmaller{\mathcal{M},\sigma}:=\arg\max_{f\in\mathcal{M}}\mathcal{G}_\sigma(f)$
serves as the oracle of EGM  induced by the gain function $p_\mathsmaller{\sigma}$. With the above notations, the following fundamental questions regarding theoretical assessments of EGM arise naturally: 

\smallskip 
\begin{description}
\item [Question 1.] \textit{
Which gain function $p_\mathsmaller{\sigma}$ should one choose}?
\smallskip 

This question is somewhat similar to the one that was proposed in the context of ERM, namely, which loss function one should choose for ERM, and was investigated in \cite{rosasco2004loss,steinwart2007compare} in some scenarios. It is generally accepted that each loss function has its own merits in learning and the choice of the loss function for ERM may rest upon the learning task confronted. For instance, the square loss may be chosen if one is interested in mean regression; the least absolute deviation loss may be preferred in performing median regression; while the check loss may be a good option for quantile regression. Likewise, in the context of EGM, while a general answer to this question is not obtainable, the choice of the gain function may also need to be discussed case-by-case. For instance, the Gaussian gain function may be adopted for robust mean regression; the Laplace gain function may be used to perform median regression robustly; while the asymmetry Laplace gain function may be utilized for robust quantile regression.

\item[Question 2.] 
\textit{What is the oracle $f_\mathsmaller{\mathcal{M},\sigma}$}?
\smallskip 

Clearly, the oracle $f_\mathsmaller{\mathcal{M},\sigma}$ is defined in association with the parameter $\sigma$. Different $\sigma$ values may lead to different oracles, which together with the non-concavity of EGM, promotes barriers to the characterization of $f_\mathsmaller{\mathcal{M},\sigma}$. It would be interesting to give a full characterization of  $f_\mathsmaller{\mathcal{M},\sigma}$ under various circumstances. For instance, regarding Gaussian EGM, some efforts were made towards this direction in \cite{fenglearning}. However, due to the dependence on the parameter $\sigma$, the oracle $f_\mathsmaller{\mathcal{M},\sigma}$ may be far from the underlying truth function $f^\star$ that one intends to approach and so may not be much informative. In particular, characterizing the oracle $f_{\mathcal{M},\sigma}$ and its relation with $f^\star$ may be much involved. This same situation also arises when seeking an answer to the following fundamental question.

\item[Question 3.]
\textit{How to bound the excess generalization gain $\mathcal{G}_\sigma(f_\mathsmaller{\mathcal{M},\sigma})-\mathcal{G}_\sigma(f_{\mathbf{z},\sigma})$}? \textit{Whether the convergence of the excess generalization gain $\mathcal{G}_\sigma(f_\mathsmaller{\mathcal{M},\sigma})-\mathcal{G}_\sigma(f_{\mathbf{z},\sigma})$ towards $0$ implies the convergence of  $\|f_{\mathbf{z},\sigma}-f^\star\|_{2,\rho}^2$}?

\smallskip 
Following the clue of learning theory studies on ERM, the above questions also arise naturally. However, as mentioned above, due to the introduction of the parameter $\sigma$, the oracle $f_\mathsmaller{\mathcal{M},\sigma}$ may drift away from $f^\star$. In this case, bounding the excess generalization gain $\mathcal{G}_\sigma(f_\mathsmaller{\mathcal{M},\sigma})-\mathcal{G}_\sigma(f_{\mathbf{z},\sigma})$ may again not be much informative and its convergence may not imply the closeness between $f_{\mathbf{z},\sigma}$ and $f^\star$. In fact, following the study in \cite{steinwart2007compare}, under some stringent assumptions on the noise variable $\varepsilon$, one may conclude that $f_{\mathcal{M},\sigma}$ is the same as $f^\star$. However, given that in the machine learning context distribution-free learning is preferred, we prefer not to impose such assumptions on the noise.

\end{description}

In what follows, we shall make efforts to address Questions 2 and 3 above. To this end, recall that the purpose of EGM is to learn the truth function $f^\star$. Though, the target hypothesis $f_{\mathcal{M},\sigma}$ may vary due to different choices of the $\sigma$ values. Therefore, what really matters here is the location function $f^\star$ rather than the target hypothesis $f_{\mathcal{M},\sigma}$. This inspires us to directly take $f^\star$ as the target hypothesis and redefine the \textit{excess generalization gain} of $f_{\mathbf{z},\sigma}$ as $\mathcal{G}_\sigma(f^\star)-\mathcal{G}_\sigma(f_{\mathbf{z},\sigma})$. With this redefinition, our main concerns in EGM based learning are then switched to the following ones: (1) Whether the excess generalization gain $\mathcal{G}_\sigma(f^\star)-\mathcal{G}_\sigma(f_{\mathbf{z},\sigma})$ decays to zero? (2) Whether $\mathcal{G}_\sigma(f_{\mathbf{z},\sigma}) \to \mathcal{G}_\sigma(f^\star)$ implies   $f_{\mathbf{z},\sigma}\to f^\star$?

\subsection{Distribution-Free Learning with EGM}
We first investigate learning performance of EGM estimators in a distribution-free setup, where distributional assumptions on the noise are absent while certain moment conditions may be imposed. To this end, we first introduce two assumptions, one on the capacity of $\mathcal{H}$ and the other on the tail behavior of the distribution of $Y$.

\begin{assumption}\label{capacity_assumption}
	$\mathcal H\subset \mathcal C(\mathcal X)$ and there exist positive constants $q$ and $c$ such that
	$$\log\mathcal{N}(\mathcal{H},\eta)\leq c \eta^{-q},\,\, \forall\,\,
	\eta>0,$$
where the covering number $\mathcal{N}(\mathcal{H},\eta)$ is defined as the minimal $k\in\mathbb{N}$ such that there exist $k$ balls in  $\mathcal C(\mathcal X)$ with centers in $\mathcal H$ and radius $\eta$ covering $\mathcal{H}$.
\end{assumption}

\begin{assumption}\label{moment_assumption}
There exists some $\epsilon>0$ such that   
$\mathbb{E}|Y|^{1+\epsilon} < +\infty. $
\end{assumption}

Assumption \ref{capacity_assumption} is typical in learning theory and is introduced here to control the complexity of the hypothesis space $\mathcal{H}$. Assumption \ref{moment_assumption} is a weak assumption on the distribution of the response variable. Note that under the boundedness assumption of $f^\star$, the finiteness of the $(1+\epsilon)$ moment condition on $Y$ is equivalent to the finiteness of that of the noise $Y-f^\star(X)$. It is rather weak as it admits the case where the noise has infinite variance.

Our first result for EGM 
is concerned with its mean regression calibration property, namely, whether $\mathcal{G}_\sigma(f_{\mathbf{z},\sigma})\to \mathcal{G}_\sigma(f^\star)$ implies $f_{\mathbf{z},\sigma}\to f^\star$. While a general answer to this question with a fixed $\sigma$ value may be negative, the following theorem tells us that some weak form of regression calibration can be obtained with an adaptive selection of $\sigma$ values.  

\begin{theorem}\label{calibration_distribution_free}
Let $f^\star=\mathbb{E}(Y|X)$ be bounded by $M$. Let Assumption \ref{moment_assumption} hold, $\sigma\geq \max\{2M,1\}$, and $p_\sigma$ be a strongly mean-calibrated gain function. For any bounded measurable function $f:\mathcal{X}\to \mathbb{R}$ with $\|f\|_\infty \leq M$, it holds that 
\begin{align*}
\left|\sigma^2\left[\mathcal{G}_\sigma(f^\star)-\mathcal{G}_\sigma(f)\right]-c_0\|f-f^\star\|_{2,\rho}^2\right|\leq c_\epsilon\sigma^{-\theta_\epsilon},
\end{align*}
where $c_0=-\psi'(0)>0$, $\theta_\epsilon=\min\{\epsilon, 2\}$, and $c_\epsilon$ is a positive constant that is independent of $f$ and will be given explicitly in the proof. Moreover, if $p_\sigma$ is exactly mean-calibrated, then the above inequality holds with $\theta_\epsilon = \epsilon$.
\end{theorem}

As Theorem \ref{calibration_distribution_free} applies to  $f_{\mathbf{z},\sigma}$, 
we say that $f_{\mathbf{z},\sigma}$ is \emph{asymptotically mean calibrated}. 
That is, when $\sigma$ is adjusted according to the sample size $n$ and its value diverges, the difference between  $\mathcal{G}_\sigma(f^\star)-\mathcal{G}_\sigma(f_{\mathbf{z},\sigma})$ and $\|f_{\mathbf{z},\sigma}-f^\star\|_{2,\rho}^2$ shrinks to $0$, yielding the calibration property.   

We next establish error bounds and convergence rates of $f_{\mathbf{z},\sigma}$. In particular, we consider two cases, namely, when the gain function $p_\sigma$ is strongly mean-calibrated and when $p_\sigma$ is exactly mean-calibrated, respectively. To this end, we introduce
\begin{align*}
f_\mathcal{H}=\arg\min_{f\in\mathcal{H}}\|f-f^\star\|_{2,\rho}^2    
\end{align*}
to characterize the approximation ability of the tuple $(\mathcal{H},\rho, L^2_{\rho_{\!_X}})$ to learn $f^\star$. When $p_\sigma$ is strongly mean-calibrated, the established  error bounds and convergence rates are as follows.

\begin{theorem}\label{error_bound_distribution_free}
Let Assumptions \ref{capacity_assumption} and \ref{moment_assumption} hold and $\sigma>\max\{2M,1\}$. Let $f_{\mathbf{z},\sigma}$ be produced by EGM \eqref{EGM_general} associated with a strong mean-calibrated gain function $p_\sigma$. For any $0<\delta<1$, with probability at least $1-\delta$, it holds that 
\begin{align}\label{error_bound_strongly}
	\|f_{\mathbf{z},\sigma}-f^\star\|_{2,\rho}^2\lesssim \|f_\mathcal{H}-f^\star\|_{2,\rho}^2+\log(2/\delta)\Psi_1(n,\epsilon,\sigma),
	\end{align}
	where 
\begin{align*}
\Psi_1(n,\epsilon,\sigma):=
\begin{cases}
\frac{1}{\sigma^\epsilon}+\frac{\sigma}{n^{1/(q+1)}},
&\,\quad\hbox{if}\quad 0<\epsilon\leq 1;\\[1ex] 
\frac{1}{\sigma^\epsilon}+\left(\frac{\sigma^{q+\frac{2\epsilon}{1+\epsilon}}}{n}\right)^{1/(q+1)}, 
&\,\quad \hbox{if}\quad 1 < \epsilon<2;\\[1ex]
\frac{1}{\sigma^2}+\left(\frac{\sigma^{q+\frac{4}{1+\epsilon}}}{n}\right)^{1/(q+1)}, 
&\,\quad \hbox{if}\quad 2\leq \epsilon <3;\\[1ex]
\frac{1}{\sigma^2}+\frac{\sigma}{n^{1/(q+1)}}, 
&\,\quad \hbox{if}\quad  \epsilon \geq 3.
\end{cases}
\end{align*}
\end{theorem}

With properly chosen $\sigma$ values, an immediate corollary
is as follows.

\begin{corollary}\label{corollary_I}
Under the assumptions of Theorem \ref{error_bound_distribution_free}, let $f^\star\in\mathcal{H}$ and $\sigma$ be chosen as $\sigma:= n^{\vartheta_1(\epsilon,q)}$, where 
\begin{align*}
\vartheta_1(\epsilon,q)=  
\begin{cases}
\frac{1}{(q+1)(\epsilon+1)}, &\quad\hbox{if}\quad 0<\epsilon\leq 1;\\[1ex]
\frac{1+\epsilon}{(1+\epsilon)(\epsilon+q+q\epsilon)+2\epsilon}, &\quad\hbox{if}\quad 1<\epsilon < 2;\\[1ex]
\frac{1+\epsilon}{(2+3q)(1+\epsilon)+\epsilon}, &\quad\hbox{if}\quad 2\leq \epsilon < 3;\\[1ex]
\frac{1}{3(q+1)}, &\quad\hbox{if}\quad \epsilon\geq 3.
\end{cases}
\end{align*}
Then for any $0<\delta<1$, with probability at least $1-\delta$, it holds that
\begin{align}\label{convergence_rates_strongly}
\|f_{\mathbf{z},\sigma}-f^\star\|_{2,\rho}^2\lesssim \log(2/\delta) n^{-\theta_\epsilon\vartheta_1(\epsilon, q)}, 
\end{align}
where $\theta_\epsilon=\min\{\epsilon,2\}$.
\end{corollary}

When $p_\sigma$ is exactly mean-calibrated, improved error bounds and rates can be established.

\begin{theorem}\label{error_bound_distribution_free_II}
Under the assumptions of Theorem \ref{error_bound_distribution_free}, we further assume that $p_\sigma$ is exactly mean-calibrated. Then for any $0<\delta<1$, with probability at least $1-\delta$, it holds that 
\begin{align}\label{error_bound_exact}
	\|f_{\mathbf{z},\sigma}-f^\star\|_{2,\rho}^2\lesssim \|f_\mathcal{H}-f^\star\|_{2,\rho}^2+\log(2/\delta)\Psi_2(n,\epsilon,\sigma),
	\end{align}
	where 
	\begin{align*}
	\Psi_2(n,\epsilon,\sigma):=
	\begin{cases}
	\frac{1}{\sigma^\epsilon}+\frac{\sigma}{n^{1/(q+1)}},
	&\,\quad\hbox{if}\quad 0<\epsilon\leq 1;\\[1ex]
	\frac{1}{\sigma^\epsilon}+\left(\frac{\sigma^{q+\frac{2\epsilon}{1+\epsilon}}}{n}\right)^{1/(q+1)}, 
	&\,\quad \hbox{if}\quad \epsilon>1.
	\end{cases}
	\end{align*}
\end{theorem}

\begin{corollary}\label{corollary_II}
Under the assumptions of Theorem \ref{error_bound_distribution_free}, let $f^\star\in\mathcal{H}$ and $\sigma$ be chosen as $\sigma:= n^{\vartheta_2(\epsilon,q)}$, where 
\begin{align*}
\vartheta_2(\epsilon,q)=  
\begin{cases}
\frac{1}{(q+1)(\epsilon+1)}, & \quad\hbox{if}\quad 0<\epsilon\leq 1;\\[1ex]
\frac{1+\epsilon}{(q+\epsilon+q\epsilon)(1+\epsilon)+2\epsilon}, & \quad\hbox{if}\quad \epsilon > 1.
\end{cases}
\end{align*}
Then for any $0<\delta<1$, with probability at least $1-\delta$, it holds that
\begin{align}\label{convergence_rates_exact}
\|f_{\mathbf{z},\sigma}-f^\star\|_{2,\rho}^2\lesssim \log(2/\delta) n^{-\epsilon\vartheta_2(\epsilon, q)}.
\end{align}
\end{corollary}

Note that in Corollaries 8 and 10, the assumption $f^\star\in\mathcal{H}$ is introduced to vanish the approximation error term $\|f_\mathcal{H}-f^\star\|_{2,\rho}^2$, which helps  with the establishment of explicit convergence rates. Proofs of Theorems \ref{error_bound_distribution_free} and \ref{error_bound_distribution_free_II} are deferred to the appendix. Results in Corollaries \ref{corollary_I} and \ref{corollary_II} are immediate from the two theorems and so their proofs are omitted. Several remarks on the theoretical results are in order here. 
\begin{itemize}
    \item First, under the $(1+\epsilon)$ moment condition, exponential type convergence rates for $f_{\mathbf{z},\sigma}$ are established  by diverging $\sigma$ values. These results  
    demonstrate that EGM estimators can deal with regression problems in the presence of heavy-tailed noise since when $0<\epsilon<1$, the noise $Y-f^\star(X)$ does not even admit finite variance. 
    \item Second, these error bounds and convergence rates explicitly tell how the scale parameter $\sigma$ in EGM influences the learnability of $f_{\mathbf{z},\sigma}$. Such an influence is weakened when $\epsilon$ goes larger, which coincides with our intuitive understanding.
    \item Third, when $p_\sigma$ is strongly mean-calibrated, the tail of the noise distribution is sufficiently light, and when functions in $\mathcal{H}$ is smooth enough, asymptotic convergence rates of type $\mathcal{O}(n^{-2/3})$ can be obtained, suggesting the existence of a bottleneck phenomenon in learning $f^\star$. While when $p_\sigma$ is exactly mean-calibrated, such asymptotic convergence rates can be up to $\mathcal{O}(n^{-1})$. These findings and comparisons indicate the advantages of exactly mean-calibrated gain functions over strongly mean-calibrated ones. 
    \item Fourth, these theoretical results can be immediately applied to the regression schemes in the motivating scenarios in Section \ref{subsec::motivating_scenarios}. Such applications bring us novel results that deepen our understanding of these well-established but not fully-understood robust regression approaches. The instantiations and applications of the above theorems and corollaries will be  detailed in Section \ref{sec::case_studies}.  
\end{itemize}

\subsection{Learning Through EGM without Misspecification}
It has been well understood that MLEs are asymptotically optimal when the likelihood function is correctly specified. Likewise, in the context of EGM, we also have a look at the case when the noise distribution is correctly specified and the gain function $p_\sigma$ results from the kernel of such a distribution.

\begin{theorem}\label{calibration_without_misspecification}
Assume that the distribution of the noise $Y-f^\star(X)$ is symmetric and is independent of $X$. Let $p_\sigma$ be the kernel of such a distribution and be symmetric and square integrable. Then
$f^\star$ is a global maximizer of the gain functional $\mathcal G_\sigma(f)$ and 
there exists an absolute constant $C_\sigma>0$ such for any bounded measurable function $f:\mathcal{X}\to \mathbb{R}$, we have
\begin{align} 
C_\sigma\|f-f^\star\|_{2,\rho}^2\leq \mathcal{G}_\sigma(f^\star)-\mathcal{G}_\sigma(f).
\label{eq:cal-1}
\end{align}
If, in addition, $p_\sigma$ is strongly mean-calibrated, then there exists an absolute constant $C_\sigma'>0$ such that 
\begin{align} 
 \mathcal{G}_\sigma(f^\star)-\mathcal{G}_\sigma(f)
 \le C_\sigma' \|f-f^\star\|_{2,\rho}^2.
\label{eq:cal-2}
\end{align}
\end{theorem}

We remark that correctly specifying the noise distribution could be a stringent and impractical requirement in real-world problems. However, the performance of EGM in this ideal situation helps better understand EGM schemes from a theoretical perspective. Theorem \ref{calibration_without_misspecification} indicates that when $\sigma$ is specified correctly so that the gain function is the kernel of the noise distribution, the resulting EGM scheme is $f^\star$-regression calibrated, that is, $\mathcal{G}_\sigma(f_{\mathbf{z},\sigma}) \to\mathcal{G}_\sigma(f^\star)$ implies $f_{\mathbf{z},\sigma} \to f^\star$ when $n\to\infty$. Note that in \eqref{eq:cal-1}, $f^\star$ denotes the underlying truth function and is not necessarily the conditional mean function, but could be broadly any location function such as the conditional median function or the conditional mode function. Moreover, the theorem also indicates that when the noise distribution is correctly specified, EGMs induced by strongly mean-calibrated gain functions are essentially equivalent to the ERM induced by the square loss while at the same time, the former ones are capable of robust regression in the absence of light-tailed noise as have been illustrated in Theorems \ref{error_bound_distribution_free} and \ref{error_bound_distribution_free_II}.

One may proceed with the establishment of error bounds and convergence rates of EGM by means of similar learning theory arguments and by recalling the regression calibration property developed in Theorem \ref{calibration_without_misspecification}. In particular, it could be also shown that faster convergence rates are obtainable due to the equivalence of strongly mean-calibrated gain functions and the square loss in this case. Details are omitted due to their great similarity to the proofs of Theorems \ref{error_bound_distribution_free} and \ref{error_bound_distribution_free_II}.

\subsection{Case Studies and Applications}\label{sec::case_studies}
The generality of the EGM framework allows us to consider specific cases by choosing specific gain functions. When the conditional mean function is of interest, we investigate above the performance of EGM estimators associated with the gain functions that are strongly mean-calibrated or exactly mean-calibrated. The usefulness of the above-established theoretical results lies in that they can be directly applied to existing well-established but yet not fully-understood robust regression schemes and provide a statistical learning assessment on them. For instance, applying these results to the four regression schemes mentioned in the motivating scenarios in Section \ref{subsec::motivating_scenarios}, we immediately obtain their error bounds and convergences rates, which are listed in Table \ref{case_studies}. While one may also apply the theoretical results to  other robust regression schemes, further exploration of the applications of the new framework and the theoretical results will be left for future research.

\begin{table}[t]
	\setlength{\tabcolsep}{10pt}
	\setlength{\abovecaptionskip}{5pt}
	\setlength{\belowcaptionskip}{5pt}
	\centering
	\captionsetup{justification=centering}
	\vspace{.5em}
	\begin{tabular}{@{\hspace{2pt}}llc@{\hspace{2pt}}}
		\toprule
	\textbf{Regression Method}	&  \textbf{Gain Function} & \textbf{Error Bounds and Rates} \\ \midrule 
			Tukey Regression &  \quad Triweight &   \eqref{error_bound_strongly} and \eqref{convergence_rates_strongly}\\ \midrule
		Truncated Least Square 
		&  \quad Epanechnikov   &  \eqref{error_bound_exact} and \eqref{convergence_rates_exact}\\ \midrule
		Geman-McClure Regression	&  \quad  Cauchy  &    \eqref{error_bound_strongly} and \eqref{convergence_rates_strongly}\\ \midrule
		 Maximum Correntropy 
		 & \quad Gaussian &  \eqref{error_bound_strongly} and \eqref{convergence_rates_strongly} \\
		\bottomrule
	\end{tabular}
	\caption{Applications of EGM Framework and Theory to Motivating Scenarios I--IV} \label{case_studies}
\end{table}

\section{Further Insights and Perspectives}\label{sec::lessons_insights}
In this section, we provide further insights and perspectives by showing that the newly developed EGM framework enables us to devise more new bounded nonconvex loss functions. As further comparisons with ERM, we also stress that, in addition to the minimum distance estimation interpretation, the adaptiveness and the boundedness of gain functions differentiate EGM from ERM.

\subsection{Devising New Bounded Nonconvex Losses from Gain Functions}
As examplified in Section \ref{translate_bounded_loss}, the EGM framework allows us to translate bounded nonconvex losses into gain functions. Such correspondence also allows us to reformulate gain functions into bounded nonconvex losses. Noticing the richness and versatility of gain functions, various new bounded nonconvex losses can be obtained. Here we example an interesting instantiation of the idea by introducing  \textit{generalized Tukey's loss} 
\begin{align*}
\ell_\mathsmaller{\sigma}(t)=
\begin{cases}
1-\left(1-\frac{|t|^m}{\sigma^m}\right)^n,& \quad\hbox{if}\quad|t|\leq \sigma;\\
1, & \quad \hbox{if}\quad|t|> \sigma.
\end{cases}
\end{align*}
The two power indices $m$, $n$ control the smoothness of the loss function and with  larger $m$ and $n$ values, the loss function becomes more smooth. In particular, the loss function turns to be more and more insensitive at the vicinity of $t=0$ when the $m$ and $n$ values become larger and larger. The generalized Tukey's loss can be derived from the gain function $p_\sigma(t)=\left(1-|t|^m/\sigma^m\right)^n$ and its introduction is inspired by the facts that when $m=2, n=3$, it reduces to the Tukey's biweight loss and when $m=2, n=1$, it gives the truncated square loss. One may explore other choices of $m$ and $n$ values, which leads to the following new bounded nonconvex losses:

\begin{description}

\item[Tricube loss from the tricube gain function.] 
The tricube loss is defined as 
\begin{align*}
\ell_\mathsmaller{\sigma}(t)=
\begin{cases}
1-\left(1-\frac{|t|^3}{\sigma^3}\right)^3,& \quad \hbox{if}\quad |t|\leq \sigma;\\
1, &\quad\hbox{if} \quad |t|>\sigma.
\end{cases}
\end{align*}
It can be reformulated from the following Tricube gain function
\begin{align*}
p_\mathsmaller{\sigma}(t)=\left(1-\frac{|t|^3}{\sigma^3}\right)^3\mathbb{I}_{\{|t|\leq \sigma\}},
\end{align*}
which results from the tricube smoothing kernel \cite{scott2015multivariate}.
\medskip  

\item[Quartic loss from the quartic gain function.] 
The quartic loss is defined as 
	\begin{align*}
	\ell_\mathsmaller{\sigma}(t)=
	\begin{cases}
	1-\left(1-\frac{t^2}{\sigma^2}\right)^2,&\quad\hbox{if}\quad |t|\leq \sigma;\\
	1, &\quad \hbox{if}\quad |t|>\sigma.
	\end{cases}
	\end{align*}
It can be derived from the quartic gain function
	\begin{align*} 
	p_\mathsmaller{\sigma}(t)=\left(1-\frac{t^2}{\sigma^2}\right)^2\mathbb{I}_{\{|t|\leq \sigma\}},
	\end{align*}
	which comes from the quartic smoothing kernel.
\medskip 

\item[Truncated absolute deviation loss from the triangle gain function.]
The truncated absolute deviation loss is defined as 
		\begin{align*}
	\ell_\mathsmaller{\sigma}(t)=
	\begin{cases}
	|t|,& \quad\hbox{if}\quad|t|\leq \sigma;\\ 
	\sigma, &\quad \hbox{if}\quad |t|>\sigma.
	\end{cases}
	\end{align*}
It can be derived from the following triangular gain function
	\begin{align*}
	p_\mathsmaller{\sigma}(t)=\left(1-\frac{|t|}{\sigma}\right)\mathbb{I}_{\{|t|\leq \sigma\}},
	\end{align*}
	 which results from the triangular smoothing kernel. 
	 
\end{description}

In addition, by hinging and translating the generalized Tukey's loss, one can also obtain a smoothened approximate of the $0-1$ loss for binary-valued regression, which may be of independent interest for robust classification.

\subsection{Boundedness and Adaptiveness of Gain Functions Make a Difference} 
Further to our comparisons of EGM with ERM, we now rethink, in addition to the minimum distance estimation interpretation of EGM, what else makes the differences between the two types of learning schemes.

Recall that a gain function can be translated into a bounded nonconvex loss and vice versa. As is commonly accepted, the boundedness of a loss function is essential in dealing with outliers in the response variable. On the other hand, EGM is associated with a gain function $p_\mathsmaller{\sigma}$ which serves as a surrogate of  $p_\mathsmaller{\varepsilon}$ and contains an integrated scale parameter $\sigma$. The introduction of this parameter provides flexibility and adaptiveness in learning.

In fact, apart from the boundedness of a gain function, it is its adaptiveness brought by the parameter $\sigma$ that distinguishes EGM from ERM. This could be further illustrated by using the following toy example on Gaussian EGM, where we consider the regression model
\begin{align}\label{toy_first}
y=f^\star(x)+\kappa(x)\varepsilon, 
\end{align}
where $x\sim U(0,1)$, $f^\star(x)=2\sin(\pi x)$, and $\kappa(x)=1+2x$. The noise variable is distributed as $\varepsilon \sim 0.5 N(-1,2.5^2)+0.5N(1,0.5^2)$. With simple computations, we know that the conditional mean function is 
$f_{\mathsmaller{\sf ME}}(x)=2\sin(\pi x)$, and the conditional mode function is approximately $f_{\mathsmaller{\sf MO}}(x)=2\sin(\pi x)+1+2x$.

\begin{figure}[t]
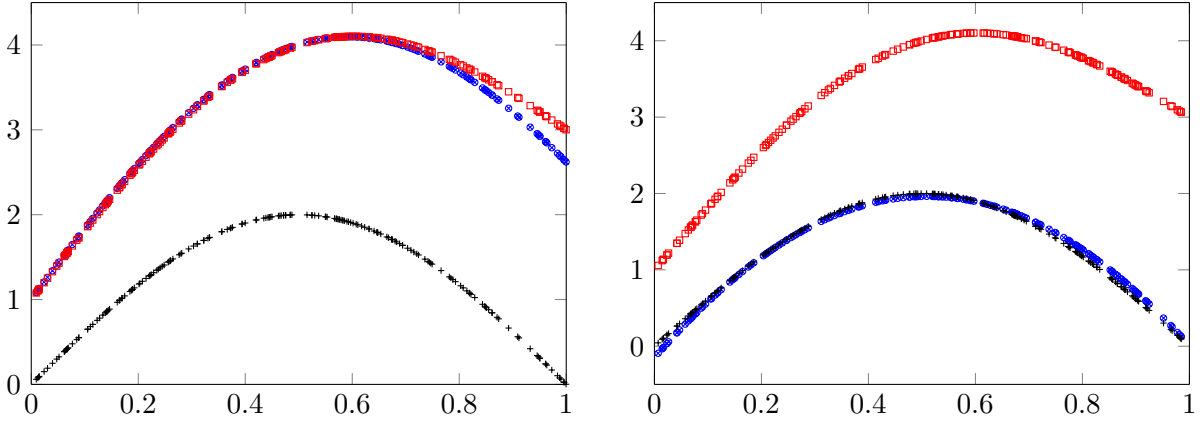
 
\tikzset{trim left=0.0 cm}
       \setlength\figureheight{6cm}
       \setlength\figurewidth{5cm}
      \begin{minipage}[b]{0.1\textwidth}
                   \input{mode_estimator.tikz}
                 \end{minipage} \quad ~~~~~~~~~~~~~~~~~~~~~~~~~~~~~~~~~~~~~~~~~~~~~~~
               \begin{minipage}[b]{0.1\textwidth}
                \input{mean_estimator.tikz}
           \end{minipage}
  \caption{In the above two panels, the dotted red curve with square marks denotes the conditional mode function $f_{\mathsmaller{\sf MO}}$ for the regression model \eqref{toy_first}. The dotted black curve with plus marks gives the conditional mean function $f_{\mathsmaller{\sf ME}}$. The dotted blue curve with $\otimes$ marks represents the learned Gaussian EGM estimator $f_{\mathbf{z},\sigma}$.}\label{toy_illustration_sigma}
\end{figure} 

In our experiment, $200$ observations are drawn from the above data-generating model for training and the size of the test set is also set to $200$. The reconstructed curve is plotted in Fig.\,\ref{toy_illustration_sigma}, in which the conditional mean function $f_{\mathsmaller{\sf ME}}$ and the conditional mode function $f_{\mathsmaller{\sf MO}}$ are also plotted. In our experiment, the hypothesis space $\mathcal{H}$ is chosen as a subset of a reproducing kernel Hilbert space induced by a Gaussian kernel, the bandwidth of which is selected through cross-validation. For the scale parameter $\sigma$ in the gain function, we set $\sigma=0.05$ in the left panel of Fig.\,\ref{toy_illustration_sigma} and $\sigma=10$ in the right panel. In the two panels, the dotted blue curves with $\otimes$ marks are the learned Gaussian EGM estimators. Clearly, from the experiments, we see that with different choices, the Gaussian EGM estimators can approach different location functions. These empirical findings together with our  theoretical results suggest that Gaussian EGM may possess more adaptiveness than ERM.

\section{Conclusion}\label{sec::conclusion}
In this paper, a framework of learning through empirical gain maximization was developed to deal with robust regression problems. The development of such a framework was inspired by several  well-established but yet not fully-understood regression schemes such as Tukey regression and Geman-McClure regression. Unlike ERM that can be traced to the framework of maximum likelihood estimation, empirical gain maximization can be interpreted from a minimum distance estimation viewpoint and thus may possess built-in robustness. To measure point-wise goodness-of-fit in regression problems, gain function was introduced. A list of gain functions was exampled and also carefully categorized. Interestingly, we showed that a variety of existing representative robust loss functions such as Tukey's biweight loss, the truncated squared loss, and the Geman-McClure loss can be reformulated as special cases of gain functions. A unified learning theory analysis was conducted to assess the performance of empirical gain maximization schemes in regression problems. The developed new framework and the conducted unified analysis not only help us better understand the existing non-convex robust regression schemes but also bring us new bounded nonconvex loss functions of the same kind.

\appendix
\section*{Appendix: Lemmas and Collected Proofs} 
\addcontentsline{toc}{section}{Appendix: Lemmas and Collected Proofs}
\addtocounter{section}{1}

In this appendix section, we provide intermediate lemmas and detailed proofs of Theorems \ref{calibration_distribution_free}, \ref{error_bound_distribution_free}, \ref{error_bound_distribution_free_II}, and \ref{calibration_without_misspecification}. Recall that for any bounded measurable function $f:\mathcal{X}\to \mathbb{R}$, $\mathcal{G}_\sigma(f)$ and  $\mathcal{G}_{\sigma,\mathbf{z}}(f)$ denote the generalization gain and the empirical gain of $f$, respectively: 
\begin{align*}
\mathcal{G}_\sigma(f)= \mathbb{E}p_\sigma(Y-f(X))\quad \hbox{and}\,\quad 
\mathcal{G}_{\sigma,\mathbf{z}}(f)=\frac{1}{n}\sum_{i=1}^n  p_\sigma(y_i-f(x_i)).   
\end{align*}
For simplification of the analysis, we introduce the scaled generalization gain $$\widetilde{\mathcal{G}}_\sigma(f) = \sigma^2 \mathcal G_\sigma (f) = \mathbb{E}\left[ \sigma^2 p_\sigma(Y-f(X)) \right] $$ and correspondingly the scaled empirical gain  
$$ \widetilde{\mathcal G}_{\sigma, \mathbf z}(f) = \sigma ^2 \mathcal{G}_{\sigma,\mathbf{z}}(f) =  \frac 1 n \sum_{i=1}^n  \sigma ^2 p_\sigma(y_i-f(x_i)).$$ We further denote $f_{\mathcal{H},\sigma}$ as the population version of $f_{\mathbf{z},\sigma}$ in $\mathcal{H}$, namely, 
\begin{align*}
f_{\mathcal{H},\sigma}=\arg\max_{f\in\mathcal{H}}\mathcal{G}_\sigma(f).   
\end{align*}
We first provide two lemmas that will be used in our proofs. 

\subsection{Lemmas}

\begin{lemma}\label{bernstein}
	Let $\xi$ be a random variable on a probability space $\mathcal{Z}$
	having variance $v$ and satisfying $|\xi-\mathbb{E}\xi|\leq
	B$ almost surely. Then for all $\varepsilon>0$, 
	\begin{align*}
	\Pr
	\left\{ \left|\frac{1}{n}\sum_{i=1}^n
	\xi(z_i)-\mathbb{E}\xi\right| > \varepsilon\right\}\leq
	2 \exp\left\{-\frac{n
		\varepsilon^2}{2(v+\frac{1}{3}B\varepsilon)}\right\}.
	\end{align*}
\end{lemma}

\begin{lemma}
\label{lem:L3}
If a gain function $p_\sigma(t)$ is strongly mean-calibrated, then its representing function $\psi(t)$ is $L_3$-Lipschitz 
with $L_3=\max(L_2+c_0, \frac{L_1}{2}).$ 
\end{lemma}

\begin{proof}
 For any $t_1, t_2\in\mathbb R,$ if both $t_1\ge 1$ and $t_2\ge 1,$ then by the fact that $\psi(t^2)$ is $L_1$-Lipschitz w.r.t. $t$, we have 
 \begin{align*}
    |\psi(t_1)-\psi(t_2)| & \le L_1 |\sqrt{t_1} - \sqrt{t_2}|
    = \frac{L_1 |t_1-t_2|}{\sqrt{t_1}+\sqrt{t_2}} \le \frac{L_1}{2} |t_1-t_2|.
 \end{align*}
 If both $t_1\le 1$ and $t_2\le 1$, since $\psi'$ exists and is $L_2$-Lipschitz continuous on [0, 1), 
we have for all $t\le 1,$
$$|\psi'(t)|\le  |\psi(t)-\psi'(0)|+|\psi'(0)| \le L_2 |t| + c_0
\le L_2+c_0.$$ Therefore,
$$|\psi(t_1)-\psi(t_2)| = \left|\int_{t_2}^{t_1} \psi'(t) {\rm d}t\right| \le  (L_2+c_0) |t_1-t_2|.$$
If $t_1<1$ and $t_2\ge 1$, then 
\begin{align*}
|\psi(t_1)-\psi(t_2)| & \le |\psi(t_1)-\psi(1)| + |\psi(1)-\psi(t_2)| 
\le (L_2+c_0)(1-t_1) + \frac{L_1}{2} (t_2-1) \\
& \le \max\left(L_2+c_0, \frac{L_1}2\right)(t_2-t_1). \end{align*}
Combining all the three cases, we obtain the desired conclusion.  
\end{proof}

\begin{lemma}\label{variance_estimate}
	Let Assumption \ref{moment_assumption} hold with some $\epsilon>0.$ Let $\sigma>\max\{2M,1\}$ and $p_\sigma$ be a strongly mean-calibrated gain function. For any measurable function $f: \mathcal{X}\rightarrow \mathbb{R}$ with $\|f\|_\infty\leq M$, we denote  $\xi(X,Y):=\sigma^2 p_\sigma(Y-f^\star(X)) - \sigma^2 p_\sigma(Y-f(X))$. Then 
	\begin{align*}
	\mathbb{E}\xi^2\leq 
	\begin{cases}
	c_1 \sigma^{1-\epsilon}, & \quad \hbox{if}\quad 0<\epsilon\leq 1;\\
	c_2 \|f-f^\star\|_{2,\rho}^{\frac{2(\epsilon-1)}{1+\epsilon}}, & \quad \hbox{if}\quad \epsilon>1,
	\end{cases}
	  \end{align*}
	where  $c_1$ and  $c_2$ 
	are absolute positive constants independent of $\sigma$ or $f$ and will be given explicitly in the proof. 
\end{lemma}
\begin{proof}
Since $\psi(t^2)$ is $L_1$-Lipschtiz w.r.t. $t,$ we have 
\begin{align} 
|\sigma^2 p_\sigma(Y-f^\star(X)) - \sigma^2 p_\sigma(Y-f(X)) |
& \le \sigma^2 L_1 \left| 
\frac{Y-f^\star(X)}{\sigma} - \frac{(Y-f(X)}{\sigma}\right|  
\nonumber\\
& = L_1 \sigma |f(X)-f^\star(X)| \le 2L_1 M \sigma . 
\label{eq:bd1}
\end{align} 
Lemma \ref{lem:L3} tells that $\psi(t)$ is $L_3$-Lipschitz w.r.t. $t$ and further implies
\begin{align} 
|\sigma^2 p_\sigma(Y-f^\star(X)) - \sigma^2 p_\sigma(Y-f(X))| & \le \sigma^2 L_3 \left| 
\left(\frac{Y-f^\star(X))}{\sigma}\right)^2 - \left(\frac{Y-f(X)}{\sigma}\right)^2\right| \nonumber\\ 
&= L_3 \left|(Y-f(X)^2 -(Y-f^\star(X))^2\right| \nonumber\\
& = L_3 \left|f(X)-f^\star(X)\right| 
\left| 2 Y -f(X) - f^\star(X)\right| \nonumber\\
& \le 2L_3 |f(X)-f^\star(X)| (|Y|+M)\nonumber\\
& \le 4 L_3 M (|Y|+M) . 
\label{eq:bd2}
\end{align}
Therefore, when $0<\epsilon\leq 1$, by \eqref{eq:bd1} and \eqref{eq:bd2}, we have
\begin{align*}
\mathbb{E}\xi^2 
&\leq (2ML_1\sigma)^{1-\epsilon} (4L_3 M)^{1+\epsilon}   \mathbb{E}\left(|Y|+M\right)^{1+\epsilon} \le c_1\sigma^{1-\epsilon},
\end{align*} 
where 
$c_1=  2^{4+2\epsilon}L_1^{1-\epsilon} 
L_3 ^{1+\epsilon} M^2 (\mathbb E|Y|^{1+\epsilon} + M^{1+\epsilon})
.$ 
When $\epsilon>1$, by \eqref{eq:bd2} and H\"{o}lder inequality, we obtain
\begin{align*}
\mathbb{E}\xi^2 
&\leq (2L_3)^2 \mathbb{E}\Big[ |f(x)-f^\star(x)|^2 (|Y|+M)^2 \Big] \\
&\leq 4 L_3^2 \|f-f^\star\|_\infty^{\frac{4}{1+\epsilon}}\mathbb{E}\left(|f(X)-f^\star(X)|^{\frac{2\epsilon-2}{1+\epsilon}}(|Y|+M)^2\right)\\
&\leq c_2 \|f-f^\star\|_{2,\rho}^{\frac{2\epsilon-2}{1+\epsilon}},
\end{align*}
where $c_2 = 8L_3^2 (2M)^{\frac 4{1+\epsilon}}
\left((\mathbb E|Y|^{1+\epsilon})^{\frac{2}{1+\epsilon}}+M^2
\right).$ This completes the proof of Lemma \ref{variance_estimate}.
\end{proof}

\subsection{Proof of Theorem \ref{calibration_distribution_free}}

\begin{proof}
For any $\sigma>\max\{2M,1\}$, let $\Omega = \{|Y|> \frac \sigma 2\}$ and $\Omega^c$ be its complement. By Markov inequality, we have 
\begin{align}\label{equ_prob}
	\Pr(\Omega)\leq \cfrac{\mathbb{E}|2Y|^{1+\epsilon}}{\sigma^{1+\epsilon}}.
\end{align} 
Recalling  the identity 
\begin{align*}
	\|f-f^\star\|_{2,\rho}^2 = 
	\mathbb E\left[ (Y-f(X))^2-(Y-f^\star(X))^2\right],
\end{align*}
we can write 
\begin{align*}
	&\Big|\left[\widetilde{\mathcal{G}}_\sigma(f^\star)-
	\widetilde{\mathcal{G}}_\sigma(f)\right]-
	c_0\|f-f^\star\|_{2,\rho}^2\Big|\\[1ex]
=\, &\Big|\mathbb E\Big(\left[\sigma^2 p_\sigma(Y-f^\star (X)) -\sigma^2 p_\sigma(Y-f(X))\right]-c_0\left[(Y-f(X))^2-(Y-f^\star(X))^2\right] \Big) \Big| \\[1ex]
\leq\, &
\mathbb E\Big( \left| \left[\sigma^2 p_\sigma(Y-f^\star (X)) -\sigma^2 p_\sigma(Y-f(X))\right]-c_0\left[(Y-f(X))^2-(Y-f^\star(X))^2\right] \right| \mathbb I_{\Omega} \Big)\\[1ex] 
&\! + \mathbb E\Big( \left| \left[\sigma^2 p_\sigma(Y-f^\star (X)) -\sigma^2 p_\sigma(Y-f(X))\right]-c_0\left[(Y-f(X))^2-(Y-f^\star(X))^2\right] \right| \mathbb I_{\Omega^c} \Big)\\[1ex]
: =\, &  Q_1 + Q_2.
	\end{align*}
By \eqref{eq:bd2} and \eqref{equ_prob} we have 
\begin{align*}
    Q_1 & \le (L_3+c_0) \mathbb E \Big[\left|(Y-f(X))^2-(Y-f^\star(X))^2\right| 
    \mathbb I_{\Omega} \Big] \\
    & \le 4M(L_3+c_0) \mathbb E\Big[ (|Y|+ M) \mathbb I_{\Omega} \Big] \\
    & \le 4M(L_3+c_0) \Big(\left( \mathbb E|2Y|^{1+\epsilon}\right)^{\frac 1 {1+\epsilon}} 
    \left(\Pr(\Omega)\right)^{\frac \epsilon{1+\epsilon}} 
    + M \Pr(\Omega)\Big) \\
    & \le 4M(L_3+c_0)(\mathbb E |2Y|^{1+\epsilon} )
    \left(\sigma^{-\epsilon} + M \sigma^{-(1+\epsilon)}\right)\\
    & \le 6 M(L_3+c_0)(\mathbb E |2Y|^{1+\epsilon} ) \sigma^{-\epsilon}.
\end{align*}
In order to bound $Q_2,$ we denote
$F_\sigma(t):= -\sigma^2 p_\sigma(t)- c_0 t^2$. 
Then \begin{align*}
	Q_2 = \mathbb E \Big [ \left|F_\sigma(Y-f(X))-F_\sigma(Y-f^\star(X))\right| \mathbb I_{\Omega^c}\Big ].
	\end{align*}
By the mean value theorem, we have 
\begin{align*}
	|F_\sigma(Y-f(X))-F_\sigma(Y-f^\star(X))|=
	|F_\sigma^\prime(a)(f(X)-f^\star(X))| \le 2 M |F_\sigma'(a)|, 
\end{align*}
	where $a$ lies between $Y-f(X)$ and $Y-f^\star(X)$ 
	and hence 
	$$|a| \le \max(|Y-f(X)|, |Y-f^\star(X)|) \le |Y|+M.$$ 
By the facts $p_\sigma(t) = \psi(t^2/\sigma^2)$ and $c_0=-\psi'(0)$, we have
	\begin{align*}
	F^\prime_\sigma(a)= -\sigma^2 p_\sigma^\prime(a)- 2ac_0 = 
	2a \Big(\psi'(0)-\psi^\prime\left(a^2/\sigma^2\right)\Big).
	\end{align*}
Recalling that $\psi^\prime(t)$ is $L_2$-Lipschitz continuous on $[0,1)$, we have 
\begin{align*}
|F^\prime_\sigma(a)|\leq \frac{2L_2|a|^3}{\sigma^2}
\leq \frac{2L_2(|Y|+M)^3}{\sigma^2}\leq 
\frac{8L_2(|Y|^3+M^3)}{\sigma^2}.
\end{align*}
Therefore, if  $\epsilon\geq 2,$ we have 
\begin{align*}
Q_2 \le 16ML_2  \left( \mathbb E |Y|^3 + M^3\right) \sigma^{-2}.
\end{align*}
If $0<\epsilon<2$, by $|Y|\le \frac \sigma 2$ on $\Omega^c,$ we obtain 
\begin{align*}
Q_2 & \le 16ML_2 \left( \left(\frac{\sigma}{2}\right)^{2-\epsilon} 
\mathbb E |Y|^{1+\epsilon} + M^3\right) \sigma^{-2} 
\le 16ML_2 \left(\mathbb E |Y|^{1+\epsilon} + M^3\right) \sigma^{-\epsilon}.
\end{align*}
Combining the estimates for $Q_1$ and $Q_2$ 
we have 
\begin{align*}
\Big|\left[\widetilde{\mathcal{G}}_\sigma(f^\star) -\widetilde{\mathcal{G}}_\sigma(f)\right]-c_0\|f-f^\star\|_{2,\rho}^2\Big|   \leq\cfrac{c_\epsilon}{\sigma^{\epsilon}},
\end{align*}
with $\theta_\epsilon=\min(\epsilon,2)$ and 
\begin{align*}
c_\epsilon= 6M(L_3+c_0)\mathbb E|2Y|^{1+\epsilon} 
+ 16ML_2(\mathbb E |Y|^{\min(1+\epsilon, 3)} +M^3).
\end{align*}
This proves the desired conclusion when $p_\sigma$ is a strongly mean-calibrated gain function.

If $p_\sigma=\psi(t^2/\sigma^2)$ is exactly mean-calibrated, i.e., $\psi^\prime(t)$ is constant on $(0, 1),$ then 
 $\psi'(t)\equiv \psi'(0)= -c_0$ on $[0,1)$ and hence
 $\psi(t) = -c_0 t$, which implies $F_\sigma(t)=0.$ Therefore, 
 $Q_2=0$ and we have 
 \begin{align*}
\Big|\left[\widetilde{\mathcal{G}}_\sigma(f^\star) -\widetilde{\mathcal{G}}_\sigma(f)\right]-c_0\|f-f^\star\|_{2,\rho}^2\Big| \le Q_1   \leq\cfrac{c^\prime_\epsilon}{\sigma^{\epsilon}},
\end{align*}
where $c^\prime_\epsilon=6M(L_3+c_0)\mathbb E|2Y|^{1+\epsilon}<c_\epsilon$. 
This completes the proof of Theorem \ref{calibration_distribution_free}.
\end{proof}

\subsection{Proof of Theorem \ref{error_bound_distribution_free}}

\medskip
\noindent\textbf{Step 1}:
We first prove that, under Assumption \ref{moment_assumption}, there are two absolute constants  $c_3^\prime$ and $c_4^\prime$ (to be defined explicitly later) such that, for any $\gamma\geq c_\epsilon\sigma^{-\theta_\epsilon}$ and $f\in\mathcal H$, there holds 
\begin{align}
\Pr\left\{ 
\frac{\big|[\widetilde{\mathcal{G}}_\sigma(f^\star)-\widetilde{\mathcal{G}}_\sigma(f)]-
	[\widetilde{\mathcal{G}}_{\sigma,\mathbf{z}}(f^\star)-\widetilde{\mathcal{G}}_{\sigma,\mathbf{z}}(f)]\big|}
{\left(\widetilde{\mathcal{G}}_\sigma(f^\star)-\widetilde{\mathcal{G}}_\sigma(f)+2\gamma\right)^{1-\zeta_\epsilon}} >
\gamma^{\zeta_\epsilon}\right\} 
\le 2\, \mathrm e^{-\Theta(n,\gamma,\sigma)},
\label{ratio_inequality}
\end{align}
where 
\begin{align*}
\zeta_\epsilon= 
\begin{cases}
\frac{1}{2},& \hbox{if}\,\, 0<\epsilon\leq 1;\\[1ex]
\frac{2}{1+\epsilon},& \hbox{if}\,\, \epsilon>1,
\end{cases} 
\quad and \quad 
\Theta(n,\gamma,\sigma)=
\begin{cases}
\frac{n\gamma}
{c_3^\prime\sigma}, & \hbox{ if }\,\,0<\epsilon\leq 1;\\[1ex]
\frac{n\gamma}{c_4^\prime\left(\sigma+\sigma^{\frac{2\theta_\epsilon}{1+\epsilon}}\right)}, & \hbox{ if }\,\,\epsilon>1.
\end{cases}
\end{align*}
To this end, for any $f\in\mathcal{H}$, consider
$
\xi=\sigma^2 p_\sigma(Y-f^\star(X))- \sigma^2 p_\sigma(Y-f(X)).
$ 
By \eqref{eq:bd1}, we know $
|\xi|\leq  2M L_1 \sigma.$ Hence $|\xi-\mathbb{E}\xi| \leq 4M L_1\sigma. $
By Lemma \ref{variance_estimate},
\begin{align*}
\hbox{var}({\xi})\leq \mathbb{E}\xi^2\leq 
	\begin{cases}
	c_1 \sigma^{1-\epsilon},& \quad \hbox{if}\quad 0<\epsilon\leq 1;\\
	c_2 \|f-f^\star\|_{2,\rho}^{\frac{2(\epsilon-1)}{1+\epsilon}}, & \quad \hbox{if}\quad \epsilon>1.
	\end{cases}
\end{align*}
By Theorem \ref{calibration_distribution_free}, when $\gamma\ge c_\epsilon \sigma^{-\theta_\epsilon},$ we have
\begin{align}
\mathbb E\xi+2\gamma & = \widetilde{\mathcal{G}}_\sigma(f^\star)-\widetilde{\mathcal{G}}_\sigma(f)+2\gamma
\geq \widetilde{\mathcal{G}}_\sigma(f^\star)-\widetilde{\mathcal{G}}_\sigma(f)+c_\epsilon\sigma^{-\theta_\epsilon}+\gamma 
\nonumber \\
& \geq c_0\|f_j-f^\star\|_{2,\rho}^2+\gamma\ge \gamma.
\label{temp_3}
\end{align}
Therefore, if $0<\epsilon\le 1$, by Lemma \ref{bernstein}, we have 
\begin{align*}
& \Pr\left\{ 
\frac{\big|[\widetilde{\mathcal{G}}_\sigma(f^\star)-\widetilde{\mathcal{G}}_\sigma(f)]-
	[\widetilde{\mathcal{G}}_{\sigma,\mathbf{z}}(f^\star)-\widetilde{\mathcal{G}}_{\sigma,\mathbf{z}}(f)]\big|}
{\sqrt{\widetilde{\mathcal{G}}_\sigma(f^\star)-\widetilde{\mathcal{G}}_\sigma(f)+2\gamma } } > \sqrt \gamma  \right\} \\
\leq\, & 2\exp\left\{-\frac{n\gamma\left( \widetilde{\mathcal{G}}_\sigma(f^\star)-\widetilde{\mathcal{G}}_\sigma(f)+2\gamma \right) } {2c_1 \sigma^{1-\epsilon} + \frac  8 3 
M L_1\sigma \sqrt{\gamma \left( \widetilde{\mathcal{G}}_\sigma(f^\star)-\widetilde{\mathcal{G}}_\sigma(f)+2\gamma \right) }} \right\}  \\
\le\, & 2\exp\left\{-\frac{n\gamma\left( \widetilde{\mathcal{G}}_\sigma(f^\star)-\widetilde{\mathcal{G}}_\sigma(f)+2\gamma \right) } {2c_1 \sigma \gamma /c_\epsilon + \frac  8 3  M
L_1 \sigma  \left( \widetilde{\mathcal{G}}_\sigma(f^\star)-\widetilde{\mathcal{G}}_\sigma(f)+2\gamma \right) } \right\}  \\
\le\,  & 2\exp\left\{-\frac{n\gamma}
{c_3\sigma}\right\},
\end{align*}
where $c_3= 2c_1/c_\epsilon + \frac  8 3  M
L_1.$ If $\epsilon>1,$ we have 
\begin{align*} 
& \Pr\left\{ 
\frac{\big|[\widetilde{\mathcal{G}}_\sigma(f^\star)-\widetilde{\mathcal{G}}_\sigma(f)]-
	[\widetilde{\mathcal{G}}_{\sigma,\mathbf{z}}(f^\star)-\widetilde{\mathcal{G}}_{\sigma,\mathbf{z}}(f)]\big|}
{\left(\widetilde{\mathcal{G}}_\sigma(f^\star)-\widetilde{\mathcal{G}}_\sigma(f)+2\gamma\right)^{1-\zeta_\epsilon}} >
\gamma^{\zeta_\epsilon}\right\} \\
\leq\, &2 \exp\left\{-\frac{n\gamma^{\frac{4}{1+\epsilon}}\left(\widetilde{\mathcal{G}}_\sigma(f^\star)-\widetilde{\mathcal{G}}_\sigma(f)+2\gamma\right)^{\frac{2\epsilon-2}{1+\epsilon}}}
{2c_2\|f_j-f^\star\|_{2,\rho}^{\frac{2(\epsilon-1)}{1+\epsilon}}+ \frac 83 ML_1 \sigma
	\gamma^{\frac{2}{1+\epsilon}}\left(\widetilde{\mathcal{G}}_\sigma(f^\star)-\widetilde{\mathcal{G}}_\sigma(f)+2\gamma\right)^{\frac{\epsilon-1}{1+\epsilon}}}\right\} \\[0.5em]
\leq\, & 2 \exp\left\{-\frac{n\gamma^{\frac{4}{1+\epsilon}}\left(\widetilde{\mathcal{G}}_\sigma(f^\star)-\widetilde{\mathcal{G}}_\sigma(f)+2\gamma\right)^{\frac{2\epsilon-2}{1+\epsilon}}}
{\left(2c_2 c_0^{-\frac{\epsilon-1}{1+\epsilon}} 
+ \frac 8 3 ML_1 \sigma
	\gamma^{\frac{2}{1+\epsilon}} \right)\left(\widetilde{\mathcal{G}}_\sigma(f^\star)-\widetilde{\mathcal{G}}_\sigma(f)+2\gamma\right)^{\frac{\epsilon-1}{1+\epsilon}}}\right\} \\[0.5em]
\le\, & 2\exp\left\{-\frac{n\gamma^{\frac{2}{1+\epsilon}}\left(\widetilde{\mathcal{G}}_\sigma(f^\star)-\widetilde{\mathcal{G}}_\sigma(f_j)+2\gamma\right)^{\frac{\epsilon-1}{1+\epsilon}}}
{2c_2c_0^{-\frac{\epsilon-1}{1+\epsilon}} \gamma^{-\frac{2}{1+\epsilon}}+ \frac 8 3 ML_1 \sigma}\right\}
\\[0.5em]
\leq\, & 2 \exp\left\{-\frac{n\gamma}{c_4\left(\sigma+\sigma^{\frac{2\theta_\epsilon}{1+\epsilon}}\right)}\right\},
\end{align*}
where $c_4:= \max\left(2c_2 c_0^{-\frac{\epsilon-1}{1+\epsilon}}
c_\epsilon^{-\frac 2 {1+\epsilon}}, \frac 8 3 ML_1\right)$ and the last inequality is due to $\gamma^{-\frac{2}{1+\epsilon}}\leq {c^{-\frac{2}{1+\epsilon}}_\epsilon}\sigma^{\frac{2\theta_\epsilon}{1+\epsilon}}$ and the fact
$\widetilde{\mathcal{G}}_\sigma(f^\star)-\widetilde{\mathcal{G}}_\sigma(f)+2\gamma\geq \gamma.$ This proves \eqref{ratio_inequality}.

\bigskip 

\noindent\textbf{Step 2}: We next show that, under Assumption \ref{moment_assumption}, the uniform concentration inequality 
\begin{align}
\Pr\left\{ \sup_{f\in\mathcal{H}}
\frac{\big|[\widetilde{\mathcal{G}}_\sigma(f^\star)-\widetilde{\mathcal{G}}_\sigma(f)]-
	[\widetilde{\mathcal{G}}_{\sigma,\mathbf{z}}(f^\star)-\widetilde{\mathcal{G}}_{\sigma,\mathbf{z}}(f)]\big|}
{\left(\widetilde{\mathcal{G}}_\sigma(f^\star)-\widetilde{\mathcal{G}}_\sigma(f)+2\gamma\right)^{1-\zeta_\epsilon}} >
4\gamma^{\zeta_\epsilon}\right\} 
\le 2 \mathcal{N}\left(\mathcal{H}, \frac{\gamma}{L_1\sigma} \right) 
\mathrm e^{-\Theta(n,\gamma,\sigma)},
\label{eq:uniform_ratio_inequality}
\end{align}
holds  for $\gamma\geq c_\epsilon\sigma^{-\theta_\epsilon}.$ 

To see this, denote  $J=\mathcal{N}(\mathcal{H},\frac{\gamma}{L_1\sigma})$ and let $\{f_j\}_{j=1}^J\subset \mathcal{H}$ be a $\frac{\gamma}{L_1\sigma}$-cover of $\mathcal{H}$. For each $1\le j\le J,$ there exists some $f\in\mathcal H$ such that $\|f-f_j\|_\infty\le \frac{\gamma} {L_1\sigma}.$ Notice that the $L_1$-Lipschitz property of $\psi(t^2)$ w.r.t. $t$ implies that 
$\sigma^2 p_\sigma(t)$ is $L_1\sigma$-Lipschitz. This in combination with \eqref{temp_3} implies
\begin{align*} 
|\widetilde{\mathcal{G}}_\sigma(f)-\widetilde{\mathcal{G}}_\sigma(f_j)|\leq \sigma L_1\|f-f_j\|_\infty\leq \gamma \le 
\gamma^{\zeta_\epsilon} \left(\widetilde{\mathcal{G}}_\sigma(f^\star)-\widetilde{\mathcal{G}}_\sigma(f)+2\gamma\right)^{1-\zeta_\epsilon}
\end{align*}
and 
\begin{align*}
|\widetilde{\mathcal{G}}_{\sigma,\mathbf{z}}(f)-\widetilde{\mathcal{G}}_{\sigma,\mathbf{z}}(f_j)|\leq \sigma L_1\|f-f_j\|_\infty\leq \gamma \le 
\gamma^{\zeta_\epsilon} \left(\widetilde{\mathcal{G}}_\sigma(f^\star)-\widetilde{\mathcal{G}}_\sigma(f)+2\gamma\right)^{1-\zeta_\epsilon}. 
\end{align*}
If 
\begin{align*}
\big|[\widetilde{\mathcal{G}}_\sigma(f^\star)-\widetilde{\mathcal{G}}_\sigma(f)]-
[\widetilde{\mathcal{G}}_{\sigma,\mathbf{z}}(f^\star)-\widetilde{\mathcal{G}}_{\sigma,\mathbf{z}}(f)]\big|	
>4\gamma^{\zeta_\epsilon}\left(\widetilde{\mathcal{G}}_\sigma(f^\star)-\widetilde{\mathcal{G}}_\sigma(f)+2\gamma\right)^{1-\zeta_\epsilon}
\end{align*}
holds for all $f\in \mathcal H,$
then for every $1\le j\le J,$ there exists $f\in\mathcal H$ such that
\begin{align*}
& \big|[\widetilde{\mathcal{G}}_\sigma(f^\star)-\widetilde{\mathcal{G}}_\sigma(f_j)]-
[\widetilde{\mathcal{G}}_{\sigma,\mathbf{z}}(f^\star)-\widetilde{\mathcal{G}}_{\sigma,\mathbf{z}}(f_j)]\big| \\
  \ge\ & \big|[\widetilde{\mathcal{G}}_\sigma(f^\star)-\widetilde{\mathcal{G}}_\sigma(f)]-[\widetilde{\mathcal{G}}_{\sigma,\mathbf{z}}(f^\star)-\widetilde{\mathcal{G}}_{\sigma,\mathbf{z}}(f)]\big|  - \left| \widetilde{\mathcal G}_\sigma(f) - \widetilde{\mathcal G}_\sigma(f_j)\right| - \left| \widetilde{\mathcal G}_{\sigma,\mathbf z} (f) - \widetilde{\mathcal G}_{\sigma,\mathbf z}(f_j)\right| \\
 >\  &  2\gamma^{\zeta_\epsilon}\left(\widetilde{\mathcal{G}}_\sigma(f^\star)-\widetilde{\mathcal{G}}_\sigma(f)+2\gamma\right)^{1-\zeta_\epsilon}  \ge \gamma^{\zeta_\epsilon} \left(\widetilde{\mathcal{G}}_\sigma(f^\star)-\widetilde{\mathcal{G}}_\sigma(f_j)+2\gamma\right)^{1-\zeta_\epsilon},
\end{align*}
where the last inequality used the estimation 
\begin{align*}
\widetilde{\mathcal{G}}_\sigma(f^\star)-\widetilde{\mathcal{G}}_\sigma(f_j)+2\gamma&=
\left (\widetilde{\mathcal{G}}_\sigma(f^\star)-\widetilde{\mathcal{G}}_\sigma(f)) + 2\gamma \right) + \left( \widetilde{\mathcal G}_\sigma(f) - \widetilde{\mathcal G}_\sigma(f_j)  \right) \\
&\leq \left (\widetilde{\mathcal{G}}_\sigma(f^\star)-\widetilde{\mathcal{G}}_\sigma(f)) + 2\gamma \right) + \gamma \\  
&\leq 2(\widetilde{\mathcal{G}}_\sigma(f^\star)-\widetilde{\mathcal{G}}_\sigma(f)+2\gamma).
\end{align*}
This proves 
\begin{align*} 
&\left\{ \sup_{f\in\mathcal{H}}
\frac{\big|[\widetilde{\mathcal{G}}_\sigma(f^\star)-\widetilde{\mathcal{G}}_\sigma(f)]-
	[\widetilde{\mathcal{G}}_{\sigma,\mathbf{z}}(f^\star)-\widetilde{\mathcal{G}}_{\sigma,\mathbf{z}}(f)]\big|}
{\left(\widetilde{\mathcal{G}}_\sigma(f^\star)-\widetilde{\mathcal{G}}_\sigma(f)+2\gamma\right)^{1-\zeta_\epsilon}} >
4\gamma^{\zeta_\epsilon}\right\}  \\[1ex]
\subset & \bigcup_{j=1}^J 
\left\{
\frac{\big|[\widetilde{\mathcal{G}}_\sigma(f^\star)-\widetilde{\mathcal{G}}_\sigma(f_j)]-
	[\widetilde{\mathcal{G}}_{\sigma,\mathbf{z}}(f^\star)-\widetilde{\mathcal{G}}_{\sigma,\mathbf{z}}(f_j)]\big|}
{\left(\widetilde{\mathcal{G}}_\sigma(f^\star)-\widetilde{\mathcal{G}}_\sigma(f_j)+2\gamma\right)^{1-\zeta_\epsilon}} >
\gamma^{\zeta_\epsilon}\right\}
\end{align*} 
and the desired uniform concentration inequality 
\eqref{eq:uniform_ratio_inequality} follows immediately from \eqref{ratio_inequality}.

\bigskip 

\noindent\textbf{Step 3}: When Assumption \ref{capacity_assumption} holds, the uniform concentration inequality \eqref{eq:uniform_ratio_inequality} becomes
\begin{align*}
\Pr\left\{ \sup_{f\in\mathcal{H}}
\frac{\big|[\widetilde{\mathcal{G}}_\sigma(f^\star)-\widetilde{\mathcal{G}}_\sigma(f)]-
	[\widetilde{\mathcal{G}}_{\sigma,\mathbf{z}}(f^\star)-\widetilde{\mathcal{G}}_{\sigma,\mathbf{z}}(f)]\big|}
{\left(\widetilde{\mathcal{G}}_\sigma(f^\star)-\widetilde{\mathcal{G}}_\sigma(f)+2\gamma\right)^{1-\zeta_\epsilon}} >
4\gamma^{\zeta_\epsilon}\right\} 
\le 2 \exp\left\{ \frac {cL_1^q \sigma^q}{\gamma^q} -\Theta(n,\gamma,\sigma)\right\}.
\end{align*}
For any $0<\delta<1,$ 
let $$2 \exp\left\{ \frac {cL_1^q \sigma^2}{\gamma^q} -\Theta(n,\gamma,\sigma)\right\} = \delta,$$
or equivalently 
$$ \Theta(n,\gamma,\sigma) - \frac {cL_1^q \sigma^2}{\gamma^q} - \log(2/\delta) = 0.$$
By  Lemma 7.2 in \cite{cucker2007learning}, the equation has a unique positive solution $\gamma^\star$ that satisfies
\begin{align*}
\gamma^\star\lesssim 
\begin{cases}
\log\left(\frac{2}{\delta}\right)\frac{\sigma}{n^{1/(q+1)}}, & \hbox{if } 0<\epsilon\le 1;\\[1em] 
\log\left(\frac{2}{\delta}\right)\left(\frac{\sigma^{q+1}+\sigma^{q+\frac{2\theta_\epsilon}{1+\epsilon}}}{n}\right)^{1/(q+1)}, & \hbox{if } \epsilon>1.
\end{cases}
\end{align*}
Let $\gamma_0=\max({c_\epsilon}{\sigma^{-\theta_\epsilon}}, \gamma^\star)$. The uniform concentration inequality tells that
\begin{align*} 
\big|[\widetilde{\mathcal{G}}_\sigma(f^\star)-\widetilde{\mathcal{G}}_\sigma(f)]-
	[\widetilde{\mathcal{G}}_{\sigma,\mathbf{z}}(f^\star)-\widetilde{\mathcal{G}}_{\sigma,\mathbf{z}}(f)]
\le 4 \gamma_0^{\zeta_\epsilon}{\left(\widetilde{\mathcal{G}}_\sigma(f^\star)-\widetilde{\mathcal{G}}_\sigma(f)+2\gamma_0\right)^{1-\zeta_\epsilon}} 
\end{align*}
holds for all $f\in \mathcal H$ with probability at least $1-\delta$.
Applying Young's inequality, 
\begin{align}
\left| [\widetilde{\mathcal{G}}_\sigma(f^\star)-\widetilde{\mathcal{G}}_\sigma(f)]-
[\widetilde{\mathcal{G}}_{\sigma,\mathbf{z}}(f^\star)-\widetilde{\mathcal{G}}_{\sigma,\mathbf{z}}(f)] \right|- \frac{1}{2}[\widetilde{\mathcal{G}}_\sigma(f^\star)-\widetilde{\mathcal{G}}_\sigma(f)]\lesssim \gamma_0 
\label{eq:allf}
\end{align}
holds for all $f\in \mathcal H$ with probability at least $1-\delta$.
Applying \eqref{eq:allf} particularly to $f_{\mathbf z, \sigma} $ and $f_{\mathcal H, \sigma},$ we have 
\begin{align}
[\widetilde{\mathcal{G}}_\sigma(f^\star)-\widetilde{\mathcal{G}}_\sigma(f_{\mathbf{z},\sigma})]-
[\widetilde{\mathcal{G}}_{\sigma,\mathbf{z}}(f^\star)-\widetilde{\mathcal{G}}_{\sigma,\mathbf{z}}(f_{\mathbf{z},\sigma})]- \frac{1}{2}[\widetilde{\mathcal{G}}_\sigma(f^\star)-\widetilde{\mathcal{G}}_\sigma(f_{\mathbf{z},\sigma})]\lesssim \gamma_0
\label{eq:fz}
\end{align}
and 
\begin{align}
[\widetilde{\mathcal{G}}_{\sigma,\mathbf{z}}(f^\star)-\widetilde{\mathcal{G}}_{\sigma,\mathbf{z}}(f_{\mathcal H,\sigma})] - 
[\widetilde{\mathcal{G}}_\sigma(f^\star)-\widetilde{\mathcal{G}}_\sigma(f_{\mathcal H,\sigma})]
- \frac{1}{2}[\widetilde{\mathcal{G}}_\sigma(f^\star)-\widetilde{\mathcal{G}}_\sigma(f_{\mathcal H,\sigma})]\lesssim \gamma_0
\label{eq:fH}
\end{align}
with probability at least $1-\delta$.

\bigskip

\noindent\textbf{Step 4}: 
By the definition of $f_{\mathbf z, \sigma}$ we know
$\widetilde{\mathcal{G}}_{\sigma,\mathbf{z}}(f_{\mathbf{z},\sigma})\geq \widetilde{\mathcal{G}}_{\sigma,\mathbf{z}}(f_{\mathcal{H},\sigma}).$ 
Therefore,
\begin{align}
& \widetilde{\mathcal{G}}_\sigma(f^\star)-\widetilde{\mathcal{G}}_\sigma(f_{\mathbf{z},\sigma})=[\widetilde{\mathcal{G}}_\sigma(f^\star)-\widetilde{\mathcal{G}}_\sigma(f_{\mathcal{H},\sigma})]+[\widetilde{\mathcal{G}}_\sigma(f_{\mathcal{H},\sigma})-\widetilde{\mathcal{G}}_\sigma(f_{\mathbf{z},\sigma})] \nonumber\\
&\leq [\widetilde{\mathcal{G}}_\sigma(f^\star)-\widetilde{\mathcal{G}}_\sigma(f_{\mathcal{H},\sigma})]+[\widetilde{\mathcal{G}}_\sigma(f_{\mathcal{H},\sigma})-\widetilde{\mathcal{G}}_\sigma(f_{\mathbf{z},\sigma})]-[\widetilde{\mathcal{G}}_{\sigma,\mathbf{z}}(f_{\mathcal{H},\sigma})-\widetilde{\mathcal{G}}_{\sigma,\mathbf{z}}(f_{\mathbf{z},\sigma})].
\label{eq:decomp}
\end{align}
Combining \eqref{eq:decomp}, \eqref{eq:fz}, and \eqref{eq:fH}, we obtain that, for any $0<\delta<1$, with probability at least $1-\delta,$ 
\begin{align*}
\frac 12 [\widetilde{\mathcal{G}}_\sigma(f^\star)-\widetilde{\mathcal{G}}_\sigma(f_{\mathbf{z},\sigma})] - \frac 3 2 [\widetilde{\mathcal{G}}_\sigma(f^\star)-\widetilde{\mathcal{G}}_\sigma(f_{\mathcal{H},\sigma})] \lesssim \gamma _0,
\end{align*}
which implies 
\begin{align*}
 \widetilde{\mathcal{G}}_\sigma(f^\star)-\widetilde{\mathcal{G}}_\sigma(f_{\mathbf{z},\sigma}) \lesssim [\widetilde{\mathcal{G}}_\sigma(f^\star)-\widetilde{\mathcal{G}}_\sigma(f_{\mathcal{H},\sigma})] + \gamma _0.  
\end{align*}
By the definition of $f_{\mathcal H, \sigma},$ we have $\widetilde{\mathcal{G}}_\sigma(f_{\mathcal{H},\sigma})\geq \widetilde{\mathcal{G}}_\sigma(f_\mathcal{H}).$ Therefore,
\begin{align*}
 \widetilde{\mathcal{G}}_\sigma(f^\star)-\widetilde{\mathcal{G}}_\sigma(f_{\mathbf{z},\sigma}) \lesssim [\widetilde{\mathcal{G}}_\sigma(f^\star)-\widetilde{\mathcal{G}}_\sigma(f_{\mathcal{H}})] + \gamma _0.  
\end{align*}
By Theorem \ref{calibration_distribution_free}, we obtain for any $0<\delta<1$, with probability at least $1-\delta,$ 
\begin{align*}
\|f_{\mathbf{z},\sigma}-f^\star\|_{2,\rho}^2 \lesssim
\widetilde{\mathcal{G}}_\sigma(f^\star)-\widetilde{\mathcal{G}}_\sigma(f_{\mathbf{z},\sigma}) + c_\epsilon \sigma^{-\theta_\epsilon} 
\lesssim \widetilde{\mathcal{G}}_\sigma(f^\star)-\widetilde{\mathcal{G}}_\sigma(f_{\mathcal H}) + \gamma_0 
\lesssim \|f_\mathcal{H}-f^\star\|_{2,\rho}^2+ \gamma_0. \end{align*}
The proof is completed by noting that $\gamma_0\lesssim \log(2/\delta) \Psi_1(n, \epsilon, \sigma).$

\subsection{Proof of Theorem \ref{error_bound_distribution_free_II}}
Theorem \ref{error_bound_distribution_free_II} can be proved analogously to Theorem \ref{error_bound_distribution_free}. 
We omit the details.

\subsection{Proof of Theorem \ref{calibration_without_misspecification}}
\begin{proof}
	Since $p_\sigma(t)$ is the kernel of the distribution of $Y-f^\star(X),$ 
	there exists a constant $c_\sigma>0$ such that 
	\begin{equation} \int_{-\infty}^{+\infty} c_\sigma p_\sigma (t) {\rm d}t =1.
	\label{eq:norming}
	\end{equation} 
	Further by the assumption that $Y-f^\star(X)$ is symmetric and independent of $X$, we know $p_\sigma(t)$ is an even function and  the density of $Y|x$ is 
	$p_{\!_{Y|x}}(y) = c_\sigma p_\sigma(y-f^\star(x))$ for all $x\in\mathcal X.$
    Consider the function 
	$$ V(s) = \int_{-\infty}^{+\infty} 
	p_\sigma(t-s) p_\sigma(t) {\rm d}t.$$
   Then we see that 
   \begin{align*}
    \mathcal G_\sigma(f) & = \int_{\mathcal X}
    \int_{-\infty}^{+\infty} p_\sigma(y-f(x)) c_\sigma p_\sigma(y-f^\star(x)){\rm d}y
    {\rm d}\rho_{\!_X}(x) \\
    & = c_\sigma \int_{\mathcal X}
    \int_{-\infty}^{+\infty} p_\sigma(t-[f(x)-f^\star(x)])  p_\sigma(t) {\rm d}t {\rm d}\rho_{\!_X}(x) \\
    & = c_\sigma \int_{\mathcal X} V(f(x)-f^\star(x)) {\rm d}\rho_{\!_X}(x).
   \end{align*}

Let $\widehat{p_\sigma}$ denote the Fourier transform of $p_\sigma,$ defined by 
$$\widehat{p_\sigma}(\xi) = \int_{-\infty}^{+\infty} p_\sigma(t) \mathrm e ^{\mathrm i \xi t} \mathrm d t.$$ 
Since $p_\sigma$ is even, $\widehat{p_\sigma}$ must be real. 
The Plancherel theorem tells that 
$$V(s) = \frac{1}{2\pi} \int_{-\infty}^{+\infty} \left( \widehat{p_\sigma}(\xi) \right)^2 \mathrm e ^{\mathrm i \xi s} \mathrm d \xi = \frac 1{2\pi} 
\int_{-\infty}^{+\infty} \left(\widehat{p_\sigma}(\xi)\right)^2 
\cos (\xi s) \mathrm d \xi.$$
It is obvious that $V$ achieves its maximum when  $s=0$, which implies $f^\star$ is a global maximizer of $\mathcal G_\sigma(f).$ Next let us write
	\begin{align*}
	\mathcal{G}_\sigma(f^\star)-\mathcal{G}_\sigma(f) & = c_\sigma \int_{\mathcal X} V(0)- V(f(x)-f^\star(x)) \mathrm d \rho_{\!_X}(x) \\ & = \frac{c_\sigma}{2\pi} \int_\mathcal{X}\int_{-\infty}^{+\infty} \left(\widehat{p_\mathsmaller{{\sigma}}}(\xi)\right)^2 \Big(1-\cos( \xi (f(x)-f^\star(x))) \Big)\mathrm{d}\xi \mathrm{d}\rho_{\!_X}(x)\\
	&=\frac{c_\sigma}{\pi} \int_\mathcal{X}\int_{-\infty}^{+\infty}\left(\widehat{p_\mathsmaller{{\sigma}}}(\xi)\right)^2 \sin^2
	\left( \xi(f(x)-f^\star(x))\right)\mathrm{d}\xi \mathrm{d}\rho_X(x).
	\end{align*}
For any $x\in\mathcal{X}$, $|f(x)-f^\star(x)|\leq 2M$. When $|\xi|\leq \frac{\pi}{2M}$, from Jordan's inequality, 
	\begin{align*}
	\sin^2\left(\frac{\xi(f(x)-f^\star(x))}{2}\right)\geq \frac{\xi^2(f(x)-f^\star(x))^2}{\pi^2}.
	\end{align*} 
	As a result,
	\begin{align*} 
	\mathcal{G}_\sigma(f^\star)-\mathcal{G}_\sigma(f)& \geq \frac{\sigma}{\pi^3} \int_{\mathcal{X}}\int_{-\frac{\pi}{2M}}^{\frac{\pi}{2M}}\xi^2\left(\widehat{p_\mathsmaller{{\sigma}}}(\xi)\right)^2(f(x)-f^\star(x))^2\mathrm{d}\xi \mathrm{d}\rho_{\!_X}(x)\\
	&=C_\sigma\int_{\mathcal{X}}(f(x)-f^\star(x))^2\mathrm{d} \rho_{\!_X}(x),
	\end{align*}
	where  
	\begin{align*} 
	C_\sigma=\frac{c_\sigma}{\pi^3}\int_{-\frac{\pi}{2M}}^{\frac{\pi}{2M}}\xi^2\left(\widehat{p_\mathsmaller{{\sigma}}}(\xi)\right)^2\mathrm{d}\xi.
	\end{align*}
	Note that \eqref{eq:norming} implies 
	$\widehat{p_\mathsmaller{{\sigma}}}(0)= \frac 1 {c_\sigma}>0.$
	This in combination with the continuity of $\widehat{p_\mathsmaller{{\sigma}}}$ tells 
	$C_\sigma>0$. This proves \eqref{eq:cal-1}. 

    When $p_\sigma$ is strongly mean-calibrated, the monotonicity of $\psi$ and the $L_1$-Lischitz property of $\psi(t^2)$ w.r.t. $t$ implies that $p_\sigma'(t)$ exists almost everywhere and is odd, non-positive, and bounded. Note further 
    $$\int_{-\infty}^{+\infty} |p_\sigma'(t)|\mathrm dt 
    = -2\int_{0}^{+\infty} p_\sigma'(t)\mathrm dt 
    = 2p_\sigma(0).$$ We obtain that
	$$V'(s) = -\int_{-\infty}^{+\infty} 
	p'_\sigma(t-s) p_\sigma(t) {\rm d}t =  -\int_{-\infty}^{+\infty} 
	p'_\sigma(t) p_\sigma(t+s) {\rm d}t $$ 
	is $\frac{2p_\sigma(0) L_1}{\sigma}$-Lipschitz and  $V'(0)=0.$ 
Therefore, for each $x\in\mathcal X$, there exists a number $a_x$ lying between $0$ and $f(x)-f^\star(x)$ such that 
\begin{align*}  V(0)-V((f(x)-f^\star(x)) & 
	=  V'(a_x)(f(x)-f^\star(x)) 
 	= (V'(a_x)-V'(0))(f(x)-f^\star(x)) \\
 	& \le \frac{2p_{\sigma}(0) L_1}{\sigma} |a_x|| f(x)-f^\star(x)| \le   
 	 \frac{2p_\sigma(0) L_1}{\sigma} |f(x)-f^\star(x)|^2.
 \end{align*}  This implies the assertion in 
	\eqref{eq:cal-2} with $C_\sigma'=\frac{2p_\sigma(0) L_1}{\sigma}$.
\end{proof}

\section*{Acknowledgement}
The authors would like to thank the reviewers for insightful comments which helped improve the quality of this paper. This work was partially supported by the Simons Foundation Collaboration Grant \#572064 (YF) and \#712916 (QW). The email addresses of the authors are {\tt{ylfeng@albany.edu}} and {\tt{qwu@mtsu.edu}}, respectively. The two authors made equal contributions to this paper and are listed alphabetically. 

\bibliographystyle{plain}
\bibliography{FW2020c} 
\end{document}